\documentclass{article} 
\usepackage{iclr2025_conference,times}
\usepackage{graphicx}

\usepackage{amsmath,amsfonts,bm}

\def\eqref#1{equation~\ref{#1}}

\def\1{\bm{1}}

\DeclareMathAlphabet{\mathsfit}{\encodingdefault}{\sfdefault}{m}{sl}
\SetMathAlphabet{\mathsfit}{bold}{\encodingdefault}{\sfdefault}{bx}{n}

\newcommand{\R}{\mathbb{R}}

\usepackage{hyperref,booktabs}
\usepackage{url}
\usepackage{multirow}
\usepackage{enumitem}
\usepackage{cleveref}
\usepackage{siunitx}
\usepackage{amsthm,wrapfig}
\usepackage{algorithm}
\usepackage{algorithmic}
\newtheorem{theorem}{Theorem}[section]

\newtheorem{assumption}[theorem]{Assumption}
\newtheorem{definition}{Definition}[section]
\newtheorem*{theorem*}{Theorem}
\newtheorem*{remark}{Remark}

\newtheorem{corollary}{Corollary}[theorem]

\usepackage{arydshln}
\usepackage{multirow}
\definecolor{mydarkred}{RGB}{192,25,25}
\definecolor{mydarkgreen}{RGB}{25,192,25}
\definecolor{mydarkblue}{RGB}{25,25,192}

\definecolor{battleshipgrey}{rgb}{0.52, 0.52, 0.51}
\definecolor{cadet}{rgb}{0.33, 0.41, 0.47}
\usepackage{xspace}
\newcommand{\algname}[1]{{\color{cadet}\small\sf#1}\xspace}
\newcommand{\algnamesmall}[1]{{\color{cadet}\scriptsize\sf#1}\xspace}

\iclrfinalcopy

\title{Randomized Asymmetric Chain of LoRA: \\ The First Meaningful Theoretical Framework for Low-Rank Adaptation}

\author{$\text{Grigory Malinovsky}^{1,2}$\thanks{Research completed during an internship at Samsung
R\&D Institute UK.\\ Corresponding author, email: grigorii.malinovskii@kaust.edu.sa} , $\text{Umberto Michieli}^{2}$, $\text{Hasan Abed Al Kader Hammoud}^{1,2}$,\\
$\textbf{Taha Ceritli}^{2}$, $\textbf{Hayder Elesedy}^{2}$, $\textbf{Mete Ozay}^{2}$, $\textbf{Peter Richt\'arik}^{1}$  \\
\\
$^{1}$Generative AI CoE\\
KAUST\\
Thuwal, Saudi Arabia 
\And
\\
$^2$Samsung R\&D Institute \\
SRUK \\
Staines-upon-Thames, United Kingdom
}

\begin{document}

\maketitle

\begin{abstract}
Fine-tuning has become a popular approach to adapting large foundational models to specific tasks. As the size of models and datasets grows, parameter-efficient fine-tuning techniques are increasingly important. One of the most widely used methods is Low-Rank Adaptation (\algname{LoRA}), with adaptation update expressed as the product of two low-rank matrices. While \algname{LoRA} was shown to possess strong performance in fine-tuning, it often under-performs when compared to full-parameter fine-tuning (\algname{FPFT}). Although many variants of \algname{LoRA} have been extensively studied empirically, their theoretical optimization analysis is heavily under-explored. The starting point of our work is a demonstration that \algname{LoRA} and its two extensions, \algname{Asymmetric LoRA} and \algname{Chain of LoRA}, indeed encounter convergence issues. To address these issues, we propose \algname{Randomized Asymmetric Chain of LoRA} (\algname{RAC-LoRA})---a general optimization framework that rigorously analyzes the convergence rates of \algname{LoRA}-based methods. Our approach inherits the empirical benefits of \algname{LoRA}-style heuristics, but introduces several small but important algorithmic modifications which turn it into a provably convergent method. Our framework serves as a bridge between \algname{FPFT} and low-rank adaptation. We provide provable guarantees of convergence to the same solution as \algname{FPFT}, along with the rate of convergence. Additionally, we present a convergence analysis for smooth, non-convex loss functions, covering gradient descent, stochastic gradient descent, and federated learning settings. Our theoretical findings are supported by experimental results.


\end{abstract}

\section{Introduction}
Many real-world Deep Learning (DL) applications require adapting a large pre-trained model to specific tasks in order to improve its performance \citep{church2021emerging}. This process, known as fine-tuning, involves adjusting the model from its pre-trained state to better handle the nuances of particular tasks or domains. Fine-tuning is a specialized form of transfer learning, where knowledge gained during pre-training is adapted for new, specific applications \citep{vrbanvcivc2020transfer}.


\textbf{Parameter-Efficient Fine-Tuning.}
While fine-tuning all model parameters has been effective, modern models with billions of parameters pose significant challenges due to their scale. Full-parameter fine-tuning is often computationally impractical with standard resources. To address this challenge, Parameter-Efficient Fine-Tuning (PEFT) \citep{he2021towards} has emerged as a solution, focusing on updating  a subset of parameters only \citep{PCDM}, or adding task-specific modules \citep{xu2023parameter}. PEFT reduces computational costs by modifying fewer parameters or adding external modules, enabling more efficient resource use and lowering storage requirements. This approach significantly reduces both training time and computational demands, making it a practical solution for adapting large models to new tasks \citep{han2024parameter}.



\subsection{Low-Rank Adaptation (LoRA)}
One of the most popular PEFT methods is Low-Rank Adaptation (\algname{LoRA}) \citep{hu2021lora}. The core idea behind \algname{LoRA} is that fine-tuning large pre-trained models can be effectively achieved by utilizing lower-dimensional parameter spaces \citep{li2018measuring, aghajanyan2020intrinsic}. Instead of updating all parameters of a large and potentially dense matrix associated with the weights of a linear layer, \algname{LoRA} works with the product of two trainable low-rank matrices, which significantly reduces the number of parameters updated during fine-tuning. These matrices are trained such that their product is added to the pre-trained model weights.


In \algname{LoRA}  \citep{hu2021lora}, the weight adaptation is represented as the product of two low-rank matrices (and a scalar multiplier), resulting in the final model 
\begin{align*}
        W = W^0 + \frac{\alpha}{r}BA,
\end{align*}
where $ W^0 \in \mathbb{R}^{m \times n} $, $ B \in \mathbb{R}^{m \times r} $, and $ A \in \mathbb{R}^{r \times n} $. Here, $r$ and $\alpha$  respectively denote the \algname{LoRA} rank and its scaling factor. Typically, since the dimensions of (particularly deep learning) models are enormous, we have rank $ r \ll \min \{m, n\} $.
 This approach saves computational resources and 
minimizes the risk of overfitting or catastrophic forgetting \citep{biderman2024lora}. Hence, \algname{LoRA} has become a lightweight and efficient technique for adapting large models to various tasks, particularly in resource-constrained environments \citep{sun2022recent}.
It is important to note that $ W^0 $ remains fixed and does not receive updates, while $ A $ and $ B $ are optimized during the training process. The scaling factor $ \alpha $ serves as a ``step size'' for the adaptation, and it is normalized by rank $r$. The matrix \( A \) is typically initialized with random Gaussian values, while the matrix \( B \) is set to zero, ensuring \( \Delta W = 0 \) at the start of training. Alternative initialization strategies were explored by \citet{zhu2024asymmetry}.

\subsection{Chain of LoRA (COLA)}
While \algname{LoRA} offers significant computational advantages in practice, it remains less effective than full-parameter fine-tuning (\algname{FPFT}) if efficiency is not a major concern \citep{biderman2024lora}.
To balance efficiency and performance, \citet{xia2024chain} proposed an iterative method called Chain of LoRA (\algname{COLA}). Essentially, \algname{COLA} simply means the successive application of several \algname{LoRA} updates. 

Chain of LoRA (\algname{COLA}) constructs a sequence of \algname{LoRA} modules through an iterative process of parameter fine-tuning, merging, and extending. The chain length is defined by the number of optimized \algname{LoRA} modules. \algname{COLA}'s central concept involves applying \algname{LoRA} adaptations iteratively $T$ times. \algname{COLA} can be summarized as training a \algname{LoRA} module, merging the updates with the fixed parameters, reinitializing the \algname{LoRA} matrices, and repeating the process \citep{xia2024chain}. The resulting model can be represented by:
$$
    W = W^0 + \frac{\alpha}{r} \sum_{t=0}^{T-1} B^t A^t,
$$
where $A^t$ and $B^t$ indicate the low-rank matrices in the $t$-th block in the chain, which are typically initialized in the same manner as in standard \algname{LoRA}. The motivation behind \algname{COLA} is that standard \algname{LoRA} may clearly fail to find the optimal adaptation since such an adaptation may not in general be of a low rank. To address this, \algname{COLA} proposes using a sequence of low-rank matrix decompositions to approximate a middle-to-high-rank update. The hypothesis is that this sequence of updates can provide a better approximation than a single \algname{LoRA} adaptation and may be easier to optimize compared to learning the optimal adaptation from scratch. 

\section{Problem Formulation and Summary of Contributions}

\subsection{Problem Formulation}
The primary approach for training supervised machine learning models is to formulate the task as an optimization problem where the goal is to minimize a loss function, which measures the discrepancy between the model's predictions and the actual outcomes. In this work, we explore this optimization problem in the specific context of fine-tuning, where a pre-trained model is adapted to a new task or dataset, requiring efficient adjustments to its parameters to achieve better performance on the target task. In particular, we consider the model-agnostic problem formulation
\begin{align}
\label{eq:main}
 \min_{\Delta W \in \mathbb{R}^{m\times n}}   f(W^0 + \Delta W) , 
\end{align}
where \( W^0 \in \mathbb{R}^{m \times n} \) represents the parameters of a pre-trained model (or of a single linear layer, with the others being fixed), and \( \Delta W \in \mathbb{R}^{m \times n} \) denotes the adaptation term. 
The function \( f: \mathbb{R}^{m \times n} \to \R \) corresponds to the empirical loss over the adaptation dataset, or any other loss function of interest. As the total dimensionality \( m \times n \) is typically very large for deep learning models, the adaptation term $\Delta W$ needs to have a specific structure to be feasible in real-world applications. 


\subsection{No Reasonable Theory for Low-Rank Adaptation}

We claim that a satisfying theoretical understanding of prevalent fine-tuning methods based on low-rank updates, such as \algname{LoRA} and \algname{COLA}, is lacking. 

\setlist{nolistsep}
\begin{itemize}[noitemsep]
\item  First, as already noted by \citet{sun2024improving}, the \algname{LoRA} re-parameterization of the domain effectively transforms a  {\em smooth} Lipschitz  loss into a {\em non-smooth} Lipschitz  loss, which poses {\em additional} theoretical challenges to those related to proper handling of the low-rank structure of the updates. While this hints at a possible source of issues with providing a good theory for methods based on low-rank adaptation, this observation does not on its own mean that a good theory is impossible to obtain. 

\item More importantly, the  existing theoretical analysis of \algname{COLA}  \citep{xia2024chain} replaces low-rank optimization over matrices \(A\) and \(B\) with full-rank matrix optimization (\(\Delta W\)). This makes the theoretical analysis irrelevant at worst and unsatisfactory at best as it completely ignores to model and to explain the key component of \algname{LoRA}: low-rank updates.

\item Third, it is known that \algname{LoRA} can be highly sensitive to the choice of the hyper-parameters \citep{khodak2021federated, kuang2024federatedscope}. A good theory should be able to explain or remove this issue. No such theory exists, to the best of our knowledge.

\item Finally, and this is the true starting point of our exploration in this work, we observe that \algname{COLA} may simply {\em fail to converge} to the optimal solution. We give a simple example (with \(3 \times 3\) matrices) of this divergence behavior 
in Section~\ref{sec:lora_convergence_issue}. Hence, \algname{COLA} is merely a {\em heuristic}. Providing a fix is an open problem -- the problem we address in this work.

\end{itemize}

While clearly \algname{LoRA} and \algname{COLA}  are enormously useful in practice, these methods remain mere {\em heuristics} since they do not come with solid theoretical backing. This is problematic and raises valid concerns about the robustness and reliability of \algname{LoRA}-type methods in scenarios {\em beyond} current datasets, models and practice.

\subsection{Contributions}

To address the aforementioned fundamental issues of \algname{LoRA}-type heuristics, and to firmly ground the fine-tuning-via-low-rank adaptation line of work in a theoretically sound algorithmic framework,  we propose a new generic low-rank adaptation framework for which we coin the name Randomized Asymmetric Chain of \algname{LoRA} (\algname{RAC-LoRA}); see Algorithm~\ref{alg:RAC-LoRA}.

\begin{table}[t]
\centering
\caption{Summary of our theoretical convergence results fpr \algname{RAC-LoRA} for solving Problem (\ref{eq:main}) when using a specific optimizer for approximately solving the subproblem in Step 4. The results for the \algname{RAC-LoRA} + \algname{GD} combination are described in Section~\ref{sec:theory}, while the proofs can be founded in Appendix~\ref{sec:GD-proofs}. The results and proofs for all other combinations can be found in the indicated appendices.}
\begin{tabular}{llccc}
\toprule
{\bf Problem} & {\bf Fine-tuner} & \textbf{Subproblem Optimizer} & \textbf{Non-convex} & \textbf{PL} \\\midrule
(\ref{eq:main}) & \algname{RAC-LoRA} & \begin{tabular}{c} {\small Gradient Descent } \\ (\algname{GD}) \end{tabular} & \begin{tabular}{c} ${\cal O}\left(1/T\right)$ \\ 
Sec.~\ref{sec:GD-pr} \end{tabular} & \begin{tabular}{c} ${\cal O}\left(\exp(-T)\right)$ \\Sec.~\ref{sec:GD-PL}\end{tabular}    \\
\hline
(\ref{eq:main})+(\ref{eq:finite}) & \algname{RAC-LoRA} & \begin{tabular}{c} {\small Random Reshuffling} \\ (\algname{RR}) \end{tabular}  &  \begin{tabular}{c}  ${\cal O}\left(1/T^{\frac{2}{3}}\right)$  \\ Sec.~\ref{sec:RR-gen} \end{tabular} &  \begin{tabular}{c}  ${\cal O}(1/T^2)$  \\ Sec.~\ref{sec:RR-PL} \end{tabular}   \\
\hline
(\ref{eq:main}) & \algname{RAC-LoRA} & \begin{tabular}{c} {\small Stochastic Gradient Descent }\\ (\algname{SGD}) \end{tabular} & \begin{tabular}{c} ${\cal O}\left(1/T^{\frac{1}{2}}\right)$ \\ Sec.~\ref{sec:SGD-gen} \end{tabular} &  \begin{tabular}{c}  ${\cal O}\left(1/T\right)$ \\ Sec.~\ref{sec:SGD-PL}\end{tabular}  \\ 
\hline
(\ref{eq:main})+(\ref{eq:fed}) & \algname{Fed-RAC-LoRA} & \begin{tabular}{c} {\small Random Reshuffling}  \\(\algname{RR}) \end{tabular}& \begin{tabular}{c}  ${\cal O}\left(1/T^{\frac{1}{2}}\right)$ \\ Sec.~\ref{sec:Fed-gen} \end{tabular} &  \begin{tabular}{c} ${\cal O}\left(1/T\right)$ \\ Sec.~\ref{sec:Fed-PL} \end{tabular} \\ 
\bottomrule
\end{tabular}
\label{tab:theory}
\end{table}

\begin{itemize}
\item Similarly to \algname{COLA} \citep{xia2024chain}, our method is iterative: we perform a chain of low-rank updates (see Step 2 in Algorithm~\ref{alg:RAC-LoRA}). In each step of the chain, 
one matrix (e.g., $A$) is chosen randomly from a pre-defined distribution, and the other (e.g., $B$) is  trainable (see Step 3 in Algorithm~\ref{alg:RAC-LoRA}). Which of these two update matrices is chosen randomly and which one is trainable is decided a-priori, and hence our method is asymmetric in nature, similarly to \algname{AsymmLoRA} \citep{zhu2024asymmetry}. We propose two options, depending on which matrix is trainable and which one is chosen randomly: in Option 1, $A$ is trainable, and in Option 2, $B$ is trainable. 
\item 
In order to make our framework flexible, we offer a variety of strategies for updating the trainable matrix in each step of the chain. This is possible since in each such step we formulate an auxiliary optimization subproblem in the trainable matrix, and one can thus choose essentially {\em any optimizer} for approximately solving it (see Step 4 in Algorithm~\ref{alg:RAC-LoRA}). We theoretically analyze several such optimizers within our \algname{RAC-LoRA} framework, including Gradient Descent (\algname{GD}) in Appendix~\ref{sec:GD-proofs} (however, we include and describe the theorems in Section~\ref{sec:GD}),  Random Reshuffling \algname{RR} in Appendix~\ref{sec:RR}, and Stochastic Gradient Descent (\algname{SGD}) in Appendix~\ref{sec:SGD}. In the case of \algname{GD} and \algname{SGD}, just a single step of the optimizer is sufficient, and this is what our analysis accounts for. In the case of \algname{RR}, we apply a single pass over the data in a randomly reshuffled order. See Table~\ref{tab:theory} for a quick overview. Our analysis applies to the smooth nonconvex regime, in which we prove fast sublinear (i.e., ${\cal O}(1/\sqrt{T})$, ${\cal O}(1/T)$ or ${\cal O}(1/T^2)$) convergence rates to a stationary point, and fast linear (i.e., ${\cal O}(\exp(-T))$) rates to the globally optimal solution under the Polyak-Łojasiewicz (PL) condition. 

\item
The update is applied (see Step 5 in Algorithm~\ref{alg:RAC-LoRA}), and the method moves on to the next step of the chain.
\end{itemize}

\textbf{Experiments.} We apply our method to several machine learning tasks. We start from convex problems with traditional models, such as logistic and linear regression, to provide clear illustrations of our theoretical findings. 
    In addition, we present empirical analyses for multilayer perception (MLP) on MNIST and RoBERTa on the GLUE benchmark tasks \citep{wang2018glue}. See Appendix~\ref{sec:exp}.

\textbf{Federated Learning.} Furthermore, we extend our findings from the simple unstructured problem (\ref{eq:main}) to the more challenging distributed/federated problem where $f$ has the special form described in (\ref{eq:fed}); there we consider solving a distributed optimization problem via our new \algname{Fed-RAC-LoRA} method (Algorithm~\ref{alg:Fed-RAC-LoRA}). These additional results can be found in Section~\ref{sec:FL}. For illustrative purposes, we provide an analysis for \algname{RR} as the optimizer for the subproblem; see also Table~\ref{tab:theory}. Previous research \citep{sun2024improving} has shown that using a single learnable matrix in this context provides several key advantages, particularly in terms of preserving privacy, ensuring the correctness of model aggregation, and maintaining stability when adjusting the scaling factor \citep{sun2024improving}. These benefits are crucial in Federated Learning \citep{konecnyFL}, where data is distributed across multiple clients, and privacy constraints must be upheld while performing model updates. Building on this asymmetric approach, we integrate the concept of chained updates to develop \algname{Fed-RAC-LoRA}, a more robust and scalable distributed method. Our approach maintains the computational efficiency of the original \algname{RAC-LoRA} while ensuring rigorous convergence properties in the distributed setting, offering a theoretically sound method for large-scale Federated Learning scenarios.

\section{Shining Some Light on LoRA's Convergence Issues}
\label{sec:lora_convergence_issue}

In contemporary machine learning, loss function minimization is primarily accomplished using gradient-based (first-order) optimization techniques \citep{ruder2016overview}. Most advanced methods build on the vanilla Gradient Descent (\algname{GD}) in various ways, e.g., by adding support for stochastic approximation, momentum, adaptive stepsizes and more \citep{shapiro1996convergence,gower2019sgd}.

It is therefore meaningful to start our exploration of \algname{LoRA}-style methods in connection with  \algname{GD} steps. In particular, we analyze the update process of \algname{LoRA} matrices through a \algname{GD} step, focusing on the application of the chain rule of differentiation. The gradient with respect to the low-rank matrices \(B\) and \(A\) consists of two components,
\begin{align*}
    \nabla_{B,A} f(W+\frac{\alpha}{r}BA) = \begin{pmatrix}
    \nabla_A f(W+\frac{\alpha}{r}BA)\\
    \nabla_B f(W+\frac{\alpha}{r}BA)
\end{pmatrix} = \frac{\alpha}{r} \begin{pmatrix}
    \nabla B^\top f(W+\frac{\alpha}{r}BA)\\
    \nabla f(W+\frac{\alpha}{r}BA) A^\top
\end{pmatrix}.
\end{align*}
and hence the update rules for the matrices \( A \) and \( B \) are given by
\[
A^+ = A - \eta \frac{\alpha}{r} B^\top \nabla f(W + \frac{\alpha}{r}BA), \qquad B^+ = B - \eta \frac{\alpha}{r} \nabla f(W + \frac{\alpha}{r}BA) A^\top,
\]

where \(\eta>0\) is a step size, and \(A^+\) and \(B^+\) are the updated matrices. Since both \(A\) and \(B\) are trainable, the gradients are multiplied by \( B^\top \) and \( A^\top \), which adds complexity to the optimization process and complicates the interpretation of its evolution. This interaction between low-rank matrices and gradients creates a non-trivial structure that challenges rigorous analysis and may disrupt Lipschitz continuity, raising concerns about convergence guarantees. While \algname{LoRA} is effective for deep learning adaptation, a deeper understanding of this process is needed to ensure that the optimization scheme is theoretically sound.


\paragraph{Loss of Lipschitz smoothness.}  Lipschitz continuity of the gradient is a commonly invoked  assumption in the theoretical analysis of gradient-based optimization methods \citep{zhou2018fenchel, khaled2020better, demidovich2023guide}. This property ensures that the gradient does not change too rapidly, which in turn guarantees a controlled behavior of the optimization process, and plays a key role in establishing convergence rates and in providing stability guarantees for various optimization algorithms \citep{nesterov2004introductory,sun2020optimization}. A formal definition follows.
\begin{assumption}[Lipschitz Gradient]
\label{asm:L-smooth}
Function $f$ is  differentiable, and there exists $L>0$ such that
\begin{align*}
\|\nabla f(W)-\nabla f(V)\| \leq L\|W-V\|, \qquad \forall  W, V \in \mathbb{R}^{m\times n},
\end{align*}
where $\|\cdot\|$ denotes the Frobenius matrix norm, the gradient is computed w.r.t.\ the trace inner product.
\end{assumption}
However, the property of Lipschitz smoothness does not necessarily hold when applying \algname{LoRA} adaptation. Specifically, even if the original function $ f(W) $ is Lipschitz smooth, meaning that the gradient of $ f(W) $ satisfies the Lipschitz continuity condition (as stated in Assumption \ref{asm:L-smooth}), this smoothness property is generally lost when the function is expressed in the adapted form $ f(W^0 + BA) $. In particular, the function $ f(W^0 + BA) $ is not Lipschitz smooth with respect to the set of variables $\{B, A\}$ for any constant. This breakdown of smoothness is a significant limitation, as it complicates the theoretical analysis of optimization algorithms when using \algname{LoRA}. The formal proof of this result is provided in Theorem 2 of the work by \citet{sun2024improving}, highlighting the challenges in extending standard gradient-based methods to such adaptations.

\begin{figure}
\includegraphics[width=0.33\textwidth]{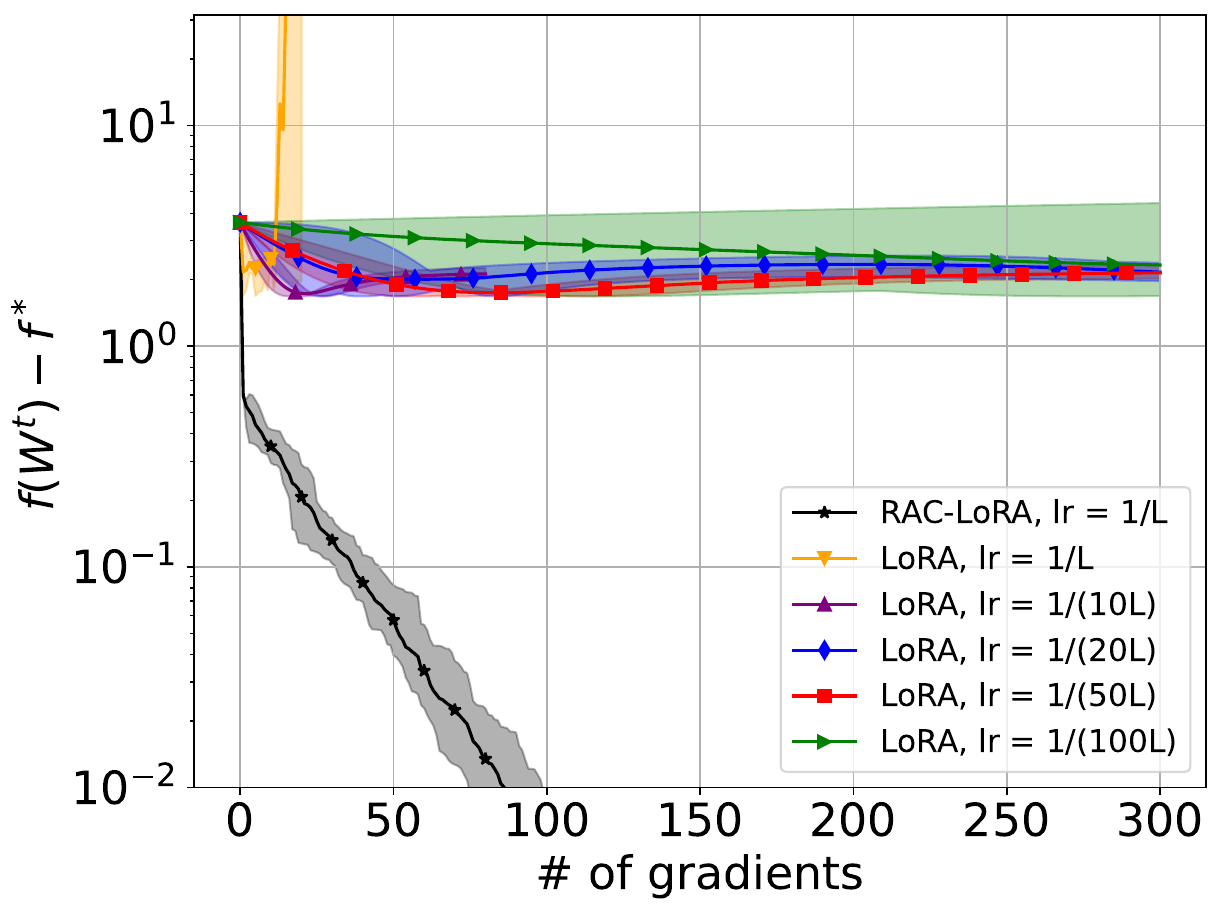}
\includegraphics[width=0.33\textwidth]{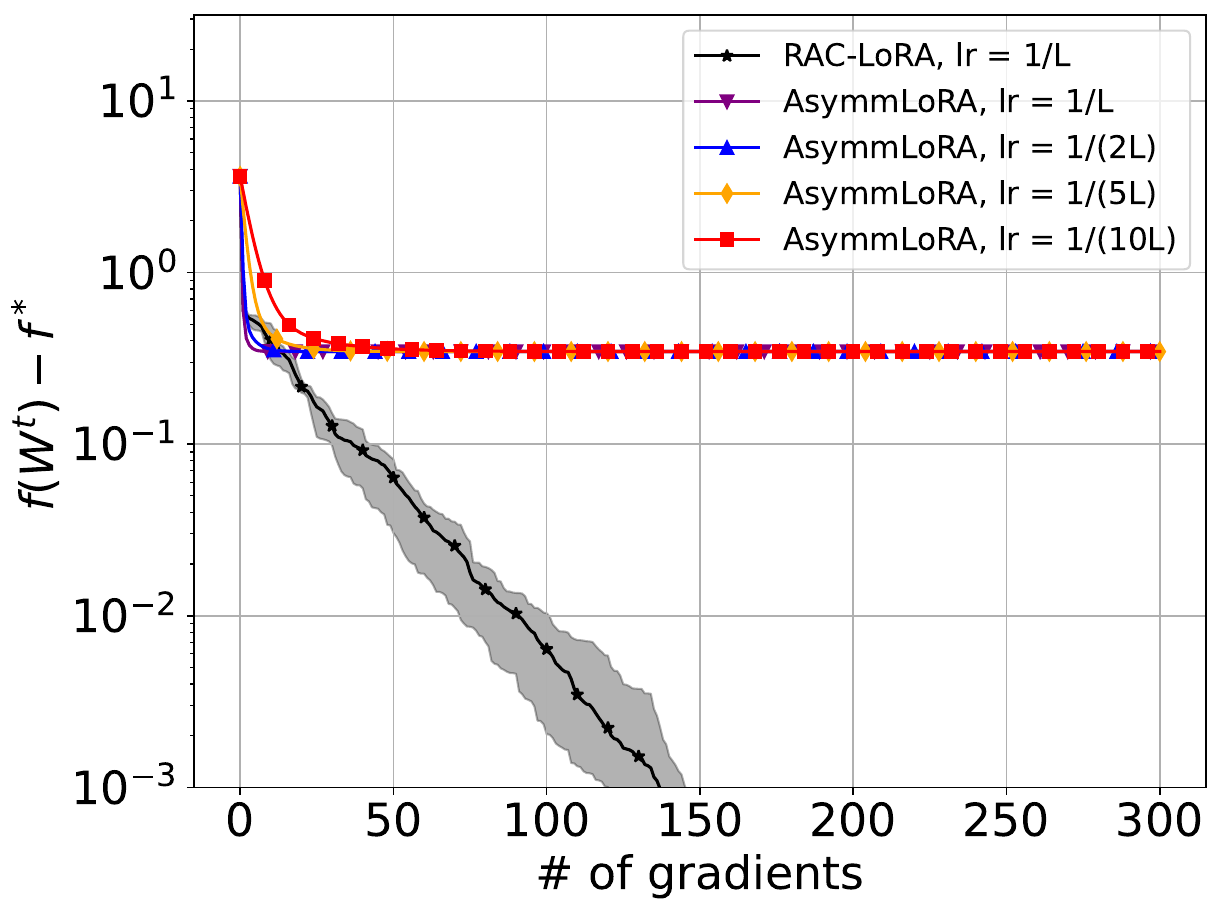}
\includegraphics[width=0.33\textwidth]{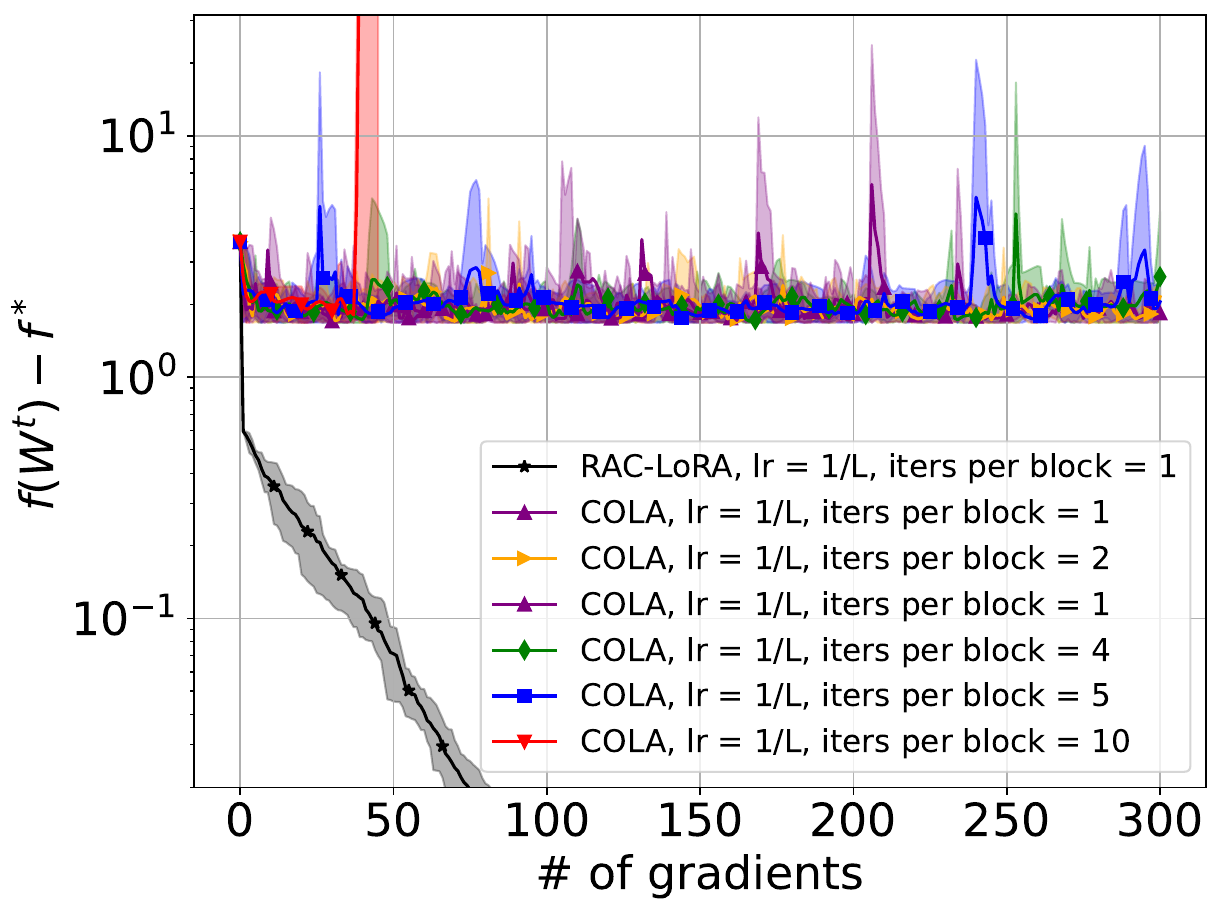}
\caption{Convergence of \algname{LoRA}, Asymmetric LoRA (\algname{AsymmLoRA}), Chain of LoRA (\algname{COLA}), and our proposed Randomized Asymmetric Chain of LoRA (\algname{RAC-LoRA}) on the problem in Equation \ref{eq:counter}.}\label{fig:counterexample}
\end{figure}

\paragraph{Numerical counterexample.} We present a clear and illustrative example demonstrating that the \algname{LoRA} and \algname{COLA} methods may not converge to the solution of the optimization problem. To illustrate this, let us consider a quadratic function of the following form:
\begin{align}
    f(x) = x^\top M x + b^\top x,
    \label{eq:counter}
\end{align}
where \( x \in \mathbb{R}^d \) is a vector of parameters, \( M \in \mathbb{R}^{d \times d} \) is a positive definite matrix, and \( b \in \mathbb{R}^d \) is a vector corresponding to the linear term. In our numerical example, we consider \(d = 9\), \(M = \operatorname{Diag}(10, 1, 1, 1, 1, 1, 1, 1, 1)\), and \(b = (1, 1, 1, 1, 1, 1, 1, 1, 1)^\top\). This function has a Lipschitz gradient (Assumption \ref{asm:L-smooth}) with \(L = 10\). We represent the vector \(x \in \mathbb{R}^9\) as a matrix \(W \in \mathbb{R}^{3 \times 3}\). In the \algname{LoRA} adaptation, we use a rank \(r = 1\) and set \(\alpha = r\).

Figure~\ref{fig:counterexample} shows experiments on \algname{LoRA}, \algname{AsymmLoRA}, our \algname{RAC-LoRA}, and \algname{COLA}. In the case of \algname{COLA}, we varied the step sizes and the number of gradients per block. Our results indicate that, when using the theoretical step size $\frac{1}{L}$, both \algname{LoRA} and \algname{COLA} may diverge, while \algname{AsymmLoRA} converges to a different stationary point. When smaller step sizes are applied to \algname{LoRA} and \algname{COLA}, these methods do converge, but to a stationary point that is significantly distant from the optimal solution. In contrast, our \algname{RAC-LoRA} converges linearly to the optimal solution without such issues. These results provide clear evidence that the choice of \algname{LoRA}-type updates has a significant impact on both the convergence and the quality of the final solution. The divergence, convergence to suboptimal points, and sensitivity to step sizes in traditional methods underscore the need for careful selection and design of update mechanisms. Our findings suggest that \algname{RAC-LoRA} offers a more reliable approach for achieving optimal solutions in the context of \algname{LoRA}-based adaptations.

\section{Randomized Asymmetric Chain of LoRA (RAC-LoRA)}

	\begin{algorithm}[t]
	\caption{Randomized Asymmetric Chain of LoRA (\algname{RAC-LoRA})}\label{alg:RAC-LoRA}
	\begin{algorithmic}[1]
		\STATE	\textbf{Parameters:}  pre-trained model $W^0 \in \R^{m\times n}$, rank $r \ll \min\{m,n\}$, learning rate $\gamma > 0$, scaling factor $\alpha>0$, chain length $T$,  sketch distribution $\mathcal{D}^B_S$ (Option 1) or $\mathcal{D}^A_S$ (Option 2).
		\FOR{$t = 0, 1, \ldots , T-1$}\do\\
          \STATE Sample a sketch matrix $$ \text{(Option 1)} \quad B^t_S\sim \mathcal{D}^B_S \qquad  \text{(Option 2)} \quad  A^t_S \sim \mathcal{D}^A_S $$
		\STATE  Using some iterative solver, approximately solve the subproblem $$ \text{(Option 1)} \quad \hat{A}^t \approx \min \limits_{A} f(W^t+ \frac{\alpha}{r} B^t_S A)\qquad \text{(Option 2)} \quad \hat{B}^t \approx \min \limits_{B} f(W^t+ \frac{\alpha}{r} B A^t_S)$$
		\STATE  Apply the update 
  $$ \text{(Option 1)} \quad W^{t+1} = W^t + \frac{\alpha}{r} B^t_S \hat{A}^t \qquad \text{(Option 2)} \quad W^{t+1} = W^t + \frac{\alpha}{r} \hat{B}^t A^t_S$$
		\ENDFOR
	\end{algorithmic}
\end{algorithm}

To address the convergence issues in \algname{LoRA} updates, we propose \algname{Randomized Asymmetric Chain of LoRA} (\algname{RAC-LoRA}). This method introduces an asymmetric \algname{LoRA} mechanism with a chain-based structure to enhance convergence while preserving model flexibility and efficiency. The method is summarized in Algorithm~\ref{alg:RAC-LoRA}.

\textbf{Description of the algorithm.} At the start of each iteration (or block), one matrix is randomly initialized and fixed throughout training, while the other remains fully trainable. This strategy prevents optimization within a restricted subspace, reducing the risk of convergence to suboptimal points. There are two configurations: freeze matrix \( B \) and train \( A \), or freeze \( A \) and train \( B \). We now formally define the sampling/sketch schemes.

\begin{definition}[Left Sketch]
\label{def:left}
By a ``left sketch'' (of rank $r$) we refer to the update rule 
$$ \Delta W = \frac{\alpha}{r} B_{S} \hat{A}, $$
where \( B_{S} \sim \mathcal{D}_B \) is sampled from some fixed distribution over matrices of dimensions \( n \times r \), and only the matrix \( \hat{A} \) is adjustable. 
\end{definition}

\begin{definition}[Right Sketch]
\label{def:right}
By a ``right sketch'' (of rank $r$) we refer to the update rule 
$$ \Delta W = \frac{\alpha}{r} \hat{B} A_{S}, $$
where \( A_{S} \sim \mathcal{D}_A \) is sampled from some fixed distribution over matrices of dimensions \( r \times m \), and only the matrix \( \hat{B} \) is adjustable. 
\end{definition}

In both sampling schemes, we update the trainable matrix over several epochs. This step effectively corresponds to training a \algname{LoRA} block within the chain, following the standard \algname{LoRA} approach. While this procedure mirrors the conventional \algname{LoRA} method, we can formally characterize it as an approximate optimization problem, allowing for a structured analysis of the training process. These procedures for both matrices can be formally expressed via
\begin{align*}
     \text{(Option 1)} \quad \hat{A}^t \approx \min \limits_{A} f(W^t+ \frac{\alpha}{r} B^t_S A)\qquad \text{(Option 2)} \quad \hat{B}^t \approx \min \limits_{B} f(W^t+ \frac{\alpha}{r} B A^t_S).
\end{align*}
Similarly to \algname{COLA}, $t$ identifies the block in the chain.
Next, we incorporate the product of the trained matrix and the sampled matrix into the current model. The merging process involves adding the product of the two matrices—one sampled and the other trained. This addition is scaled by a factor of \(\frac{\alpha}{r}\), ensuring the appropriate weighting of the update within the model:
\begin{align*}
     \text{(Option 1)} \quad W^{t+1} = W^t + \frac{\alpha}{r} B^t_S \hat{A}^t \qquad \text{(Option 2)} \quad W^{t+1} = W^t + \frac{\alpha}{r} \hat{B}^t A^t_S.
\end{align*}

\section{Theory}\label{sec:theory}



\subsection{Derivation of the update step}
\label{sec:der}
Without loss of generality, let us focus on the Left Sketch scheme (Definition~\ref{def:left}). Specifically, for each model in the chain, the update rule is given as follows:
\begin{align*}
    W^{t+1} = W^t + \frac{\alpha}{r} B^t_S \hat{A}^t. 
\end{align*}
Next, we apply the Lipschitz gradient condition (Assumption \ref{asm:L-smooth}) to the loss function $f$:
\begin{align*}
    f(U) \leq f(V)+\langle\nabla f(V), U-V \rangle+\frac{L}{2}\|U-V\|_F^2, \quad \forall U,V \in \mathbb{R}^{m\times n}
\end{align*}
Applying this with $U = W^t$, $V = B^t_S \hat{A}^t$ and $\eta\leq \frac{1}{L}$ leads to
\begin{align*}
    f(W^{t+1}) &\leq f(W^t)+\langle\nabla f(W^t), B^t_S \hat{A}^t \rangle+\frac{L}{2}\|B^t_S \hat{A}^t\|_F^2\\
    &\leq f(W^t)+\langle (B^t_S)^\top \nabla f(W^t), \hat{A}^t \rangle+\frac{1}{2\eta}\langle (B^t_S)^\top B^t_S \hat{A}^t , \hat{A}^t\rangle. 
\end{align*}
Let us minimize the left hand side term in $\hat{A}^t$, when the gradient vanishes: $
    (B^t_S)^\top \nabla f(W^t) + \frac{1}{\eta} (B^t_S)^\top (B^t_S) \hat{A}^t = 0. 
$ One such solution is given by\footnote{The dagger notation refers to the Moore-Penrose pseudoinverse.}
\begin{align*}
\hat{A}^t = -\eta  \left((B^t_S)^\top (B^t_S) \right)^\dagger (B^t_S)^\top  \nabla f(W^t),
\end{align*}
and his leads to the following gradient update:
\begin{align}
\label{eq:left_GD}
 \notag   W^{t+1} &= W^t + \frac{\alpha}{r} B^t_S \hat{A}^t = W^t -\frac{\alpha}{r} \eta  B^t_S \left((B^t_S)^\top (B^t_S) \right)^\dagger (B^t_S)^\top  \nabla f(W^t)\\
    & = W^t - \gamma H^t_B \nabla f(W^t),
\end{align}
where $H^t_B =  B^t_S \left((B^t_S)^\top (B^t_S) \right)^\dagger (B^t_S)^\top$ is projection matrix and $\frac{\alpha}{r} \eta = \gamma$. Similarly, we can obtain the update for Right Sketch scheme (Definition~\ref{def:right}):
\begin{align}
\label{eq:right_GD}
    W^{t+1} &= W^t - \gamma \nabla f(W^t) (A^t_S)^\top \left(A_S^t(A^t_S)^\top\right)^\dagger A^t_S  = W^t - \gamma \nabla f(W^t) H^t_A, 
\end{align}
where $H^t_A = (A^t_S)^\top \left(A_S^t(A^t_S)^\top\right)^\dagger A^t_S$ is also projection matrix. Notably, the scaling factor \(\frac{\alpha}{r}\) is combined with the parameter \(\eta\), allowing us to work with the effective step size \(\gamma\). This simplifies the learning process by unifying the scaling and learning rate. Using this type of update, we provide convergence results for both standard and stochastic gradient descent methods.

\subsection{Convergence results} 
\label{sec:GD}
To derive the convergence results, a key factor in our analysis is the smallest eigenvalue of the expected value of the projection matrix introduced in Section~\ref{sec:der}. This eigenvalue plays a critical role in shaping the optimization process. As we will show, a well-conditioned projection matrix—with a sufficiently large smallest eigenvalue—ensures more efficient and reliable convergence. Therefore, we make an important assumption that this smallest eigenvalue must remain strictly positive.

\begin{assumption}
\label{asm:lambda}
Consider a projection matrix $H$ generated by Left Sketch (Def.~\ref{def:left}) or Right Sketch (Def.~\ref{def:right}). Assume that the sampling distributions \(\mathcal{D}^B_S\) and \(\mathcal{D}^A_S\) are such that the smallest eigenvalue of the expected projection matrix $H$ generated by sampled matrix is positive:
\begin{align*}
    \lambda_{\min}^H = \lambda_{\min}\left[ \mathbb{E}\left[H\right] \right] >0.
\end{align*}
    
\end{assumption}

In particular, it is important to observe that the eigenvalues of the projection matrix are either zero or one, with the smallest eigenvalue being zero. However, the smallest eigenvalue of the expected value of the projection matrix can be strictly greater than zero. Additionally, it is essential to establish a lower bound for the loss function.
\begin{remark}
    Assumption \ref{asm:lambda} is easily satisfied.
    Let $H$ be the projection matrix as
    defined below~\Cref{eq:right_GD} and
    assume that the $A$ matrices are
    drawn from an isotropic distribution
    (the rows of $A$ are isotropic).
    Then $H$ is the projection onto the rank
    of $A$, which is a subspace of dimension
    $r$ distributed isotropically in $\mathbb{R}^n$.
    The matrix $\mathbb{E}[H]$ is then invariant
    under rotations, so must be a scalar multiple of the identity. By taking
    traces, one finds that $\mathbb{E}[H] = \frac{r}{n}I$ so $\lambda_{\min}^H = \frac{r}{n}$.
\end{remark}
\begin{assumption}
    Function $f$ is bounded from below by an infimum $f^\star \in \mathbb{R}$. 
    \end{assumption}
We now present the convergence result for \algname{RAC-LoRA} with Gradient Descent (\algname{GD}) updates.

\begin{theorem}
\label{thm:GD}
Let Assumptions \ref{asm:L-smooth} and \ref{asm:lambda} hold, and let the stepsize satisfy $0<\gamma\leq \frac{1}{L}$.  Then, the iterates of \algname{RAC-LoRA}  (Algorithm \ref{alg:RAC-LoRA}) with \algname{GD} updates (Equation \ref{eq:left_GD} or \ref{eq:right_GD}) satisfy
                   \begin{align*}
 \mathbb{E}\left[\left\| \nabla f(\widetilde{W}^T)\right\|^2 \right] \leq \frac{2 (f(W^0) - f^\star)}{\lambda^H_{\min} \gamma T},
   \end{align*}
where the output $\widetilde{W}^T$ is chosen  uniformly at random from $W^0, W^1,\ldots,W^{T-1}$.   
\end{theorem}
We obtain a sub-linear convergence rate, as is expected in general non-convex settings. To achieve a stronger convergence result, we employ an additional assumption: the Polyak-Łojasiewicz (PL) condition. This assumption generalizes strong convexity but applies to certain non-convex functions.

\begin{assumption}[PL-condition] 
    \label{asm:PL}
Function $f$ satisfies the Polyak-Łojasiewicz (PL) condition with parameter $\mu > 0$  if
\begin{align*}
    \frac{1}{2}\|\nabla f(W)\|^2 \geq \mu\left(f(W)-f^{\star}\right)
\end{align*}
for all $W \in \mathbb{R}^{m\times n}$,
where $f^\star = \inf f$, assumed to be finite.
\end{assumption}
Next, we establish a convergence rate for \algname{RAC-LoRA} in the Polyak-Łojasiewicz setting.

\begin{theorem}
\label{thm:PL-GD}
Let Assumptions \ref{asm:L-smooth}, \ref{asm:lambda} and  \ref{asm:PL} hold, and let the stepsize satisfy  $0<\gamma\leq \frac{1}{L}$. Then, for each $T\geq 0$, the iterates of \algname{RAC-LoRA} (Algorithm \ref{alg:RAC-LoRA}) with \algname{GD} updates (Equation \ref{eq:left_GD} or \ref{eq:right_GD})  satisfy
     \begin{align*}
         \mathbb{E}\left[ f(W^{T}) \right] -f^\star  \leq  \left(1 - \gamma \mu \lambda^H_{\min}\right)^T\left( f(W^0) - f^\star \right).
     \end{align*}
\end{theorem}

We achieved a linear convergence rate, which is significantly better than previous results; however, this improvement applies to a more limited class of functions. Importantly, we can recover the classical results of \algname{GD} by setting \(\lambda_{\min}^H = 1\), which corresponds to the full-rank scenario. 

The comprehensive analysis of different optimizers and their performance across various settings is provided in the appendix, as summarized in Table \ref{tab:theory}.

\section{Experiments} \label{sec:exp}

In this section, we explore the performance of 
\algname{RAC-LoRA} as an optimization algorithm in machine learning applications.
In~\Cref{sec:exp-convex} we validate the theoretical results in convex problems,
while in~\Cref{sec:exp-nonconvex} we evaluate the method applied to neural networks.


\subsection{Convex Optimization Problems}\label{sec:exp-convex}
\begin{figure}
\centering
\includegraphics[width=0.4\textwidth]{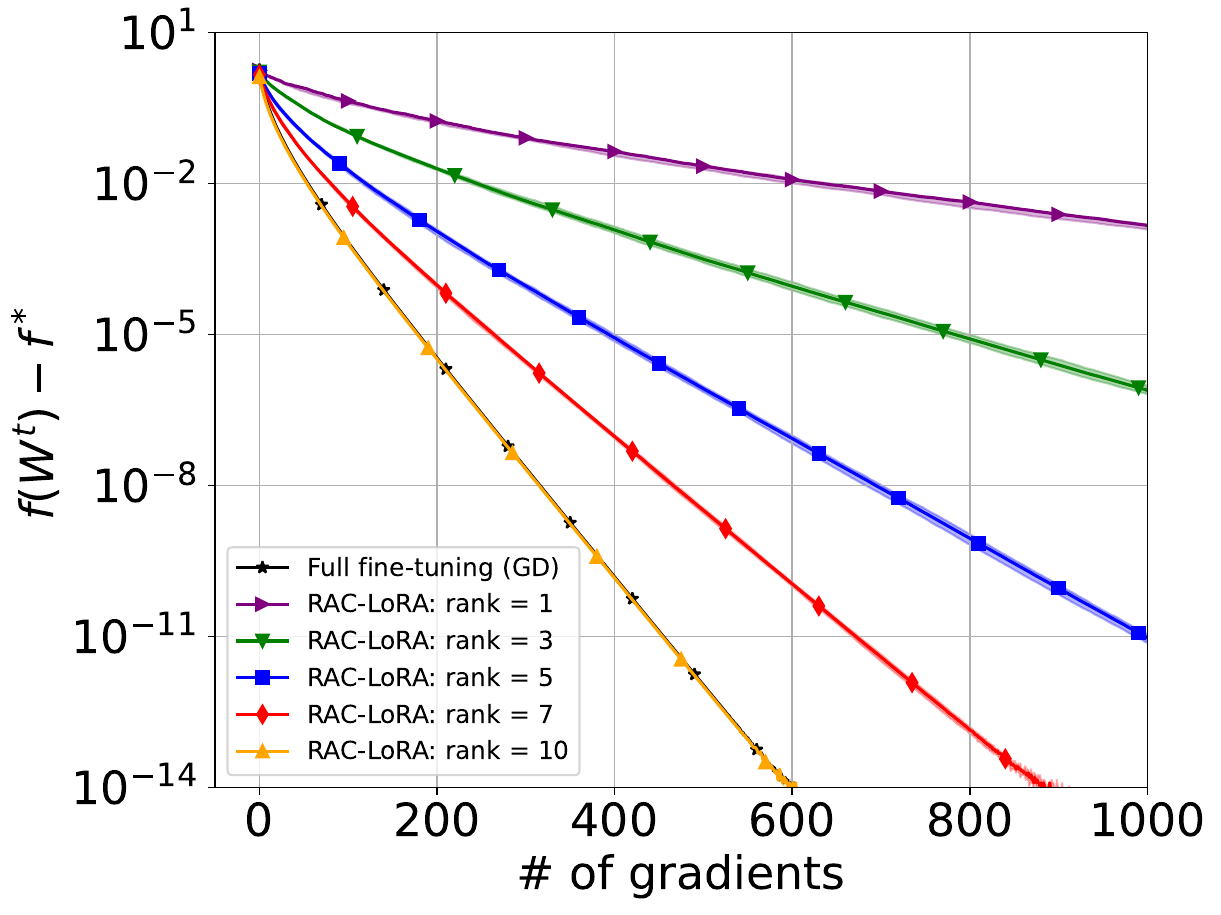}\qquad\qquad
\includegraphics[width=0.4\textwidth]{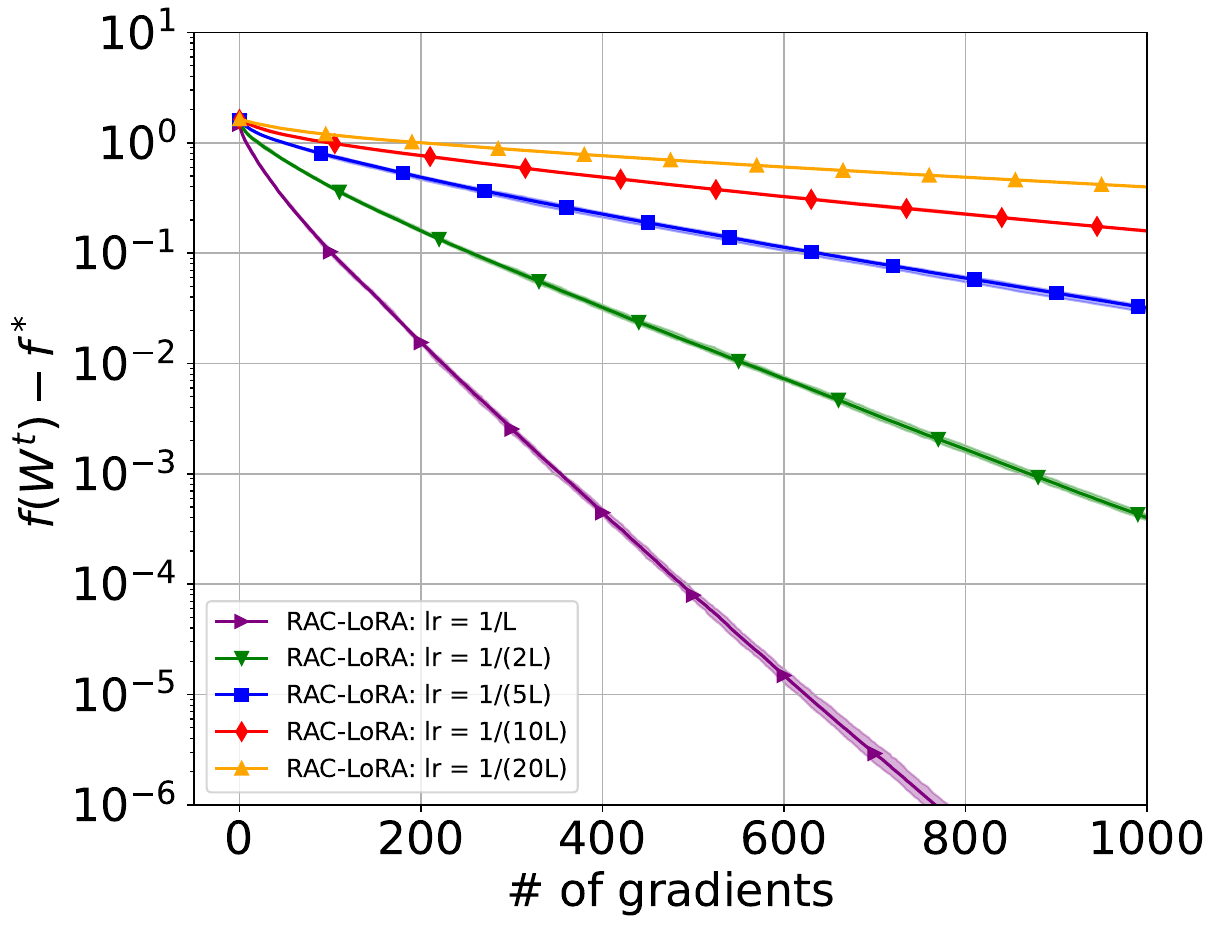}
\vspace{-0.25cm}
\caption{\algnamesmall{RAC-LoRA} convergence with varying ranks and step sizes on a linear regression problem.}\label{fig:linereg}
\end{figure}



\textbf{Linear Regression.} We conducted our analysis in a controlled setting involving linear regression with quadratic regularization applied to synthetic data. Specifically, we utilized 3,000 samples for pre-training the model and 1,000 samples for fine-tuning. In this setup, we have \(d = 100\) with weight matrices of size \(10 \times 10\), and the regularization term is set to \(0.0001\). As illustrated in Figure~\ref{fig:linereg}, the method converges for various ranks and the convergence speed is proportional to $\frac{n}{r}$, and when the rank is set to the full rank, we observe convergence identical to that of \algname{FPFT}. We remark that \algname{COLA} would suffer from the same divergence behavior as in \Cref{fig:counterexample} on this quadratic problem.

\textbf{Logistic Regression.} Analogous results for logistic regression are shown in~\Cref{sec:app:exp-convex}.



\subsection{Non-Convex Optimization Problems}\label{sec:exp-nonconvex}

Further experimental results are provided in Appendix~\ref{sec:app:nonconvex}.

\subsubsection{Results of RoBERTa on NLP Tasks}

As in prior work  \citep{zhu2024asymmetry,xia2024chain}, we evaluate 
low-rank adaptation methods for LLMs using the GLUE dataset \citep{wang2018glue}.

\textbf{Methodology.}
We fine-tuned the \texttt{roberta-base} model \citep{liu2019roberta} on four of the smallest GLUE tasks to study the behavior of low-rank methods in practical scenarios.
For the chained methods, we use a range of values for the number of chains and
epochs per chain hyperparameters.
In each experiment
we used rank 2 for the adaptations and trained using the \algname{AdamW} optimiser~\citep{loshchilov2017decoupled} with $\beta$ parameters 0.9 and 0.999,
$\epsilon=\num{1e-8}$, a learning rate of \num{4e-4} with linear schedule and a training batch size 8.

\textbf{Discussion.}
The results are presented in~\Cref{tab:main:results_roberta_base_rank=2}.
We find that \algname{RAC-LoRA} performs competitively with other low-rank adaptation methods, but
does not outperform \algname{Asymmetric LoRA} despite having greater capacity.
We expect \algname{RAC-LoRA} to outperform \algname{Asymmetric LoRA} in settings where there
is a benefit to the additional capacity, i.e., those where a full parameter
fine tune (\algname{FPFT}) is much better than \algname{Asymmetric LoRA}. 
The performance of the \algname{FPFT}
in~\Cref{tab:main:results_roberta_base_rank=2} shows that
the selected GLUE tasks do not provide such a setting. Here,
a single low-rank adaptation is already enough to obtain
performance close to that of \algname{FPFT}.
However, this intuition motivates the experiments in~\Cref{sec:exp-mnist}
where we intentionally restrict capacity of the adaptations
to isolate the effect of the chaining procedure.

\begin{table*}[htb!]
\resizebox{\textwidth}{!}{%
\centering
{
\begin{tabular}{lccccccc}
 \toprule
{{\bf Method}} & {{\bf \# Chains}} & {{\bf \# Epochs}} & \textcolor{black}{{\bf MRPC}} & \textcolor{black}{{\bf CoLA}} & \textcolor{black}{{\bf RTE}} & \textcolor{black}{{\bf STS-B}} &  {\bf Avg}\\
 \midrule

\algname{FPFT}* &  \multirow{2}{*}{1} & \multirow{2}{*}{30, 80, 80, 40} & 
90.2\textsubscript{$\pm$0.0} &
63.6\textsubscript{$\pm$0.0} &
78.7\textsubscript{$\pm$0.0} & 
91.2\textsubscript{$\pm$0.0} & 
80.9\\

\algname{LoRA}* & & & 
89.7\textsubscript{$\pm$0.7} &
63.4\textsubscript{$\pm$1.2} &
86.6\textsubscript{$\pm$0.7} & 
91.5\textsubscript{$\pm$0.2} & 
82.8\\

\midrule

\algname{LoRA} & \multirow{2}{*}{1} & \multirow{2}{*}{100}  
&
87.7\textsubscript{$\pm$0.2} &
{60.8\textsubscript{$\pm$0.2}} &
{75.2\textsubscript{$\pm$1.5}} & 
90.2\textsubscript{$\pm$0.1} & 
78.5\\
 
\algname{AsymmLoRA} & & 
&
86.9\textsubscript{$\pm$0.3} &
58.7\textsubscript{$\pm$1.0} &
71.0\textsubscript{$\pm$3.3} & 
90.4\textsubscript{$\pm$0.0} & 
76.8\\

\algname{COLA} & 10 & 10
& 
{88.0\textsubscript{$\pm$0.8}} &
59.5\textsubscript{$\pm$1.0} &
72.1\textsubscript{$\pm$0.9}&
{90.7\textsubscript{$\pm$0.2}} & 
77.6\\

\algname{RAC-LoRA} & 10 & 10
& 
87.0\textsubscript{$\pm$0.7} &
58.5\textsubscript{$\pm$0.1} &
72.3\textsubscript{$\pm$1.5}&
90.3\textsubscript{$\pm$0.0} & 
77.0\\

\bottomrule
\end{tabular}
}
}

\caption{Results with RoBERTa-base for rank 2 on tasks from the GLUE benchmark. *: results taken from the work of~\cite{hu2021lora}. We report Matthews correlation coefficient for \algname{COLA}, Pearson correlation coefficient for STS-B, and accuracy for the remaining tasks. Results are averaged over 3 seeds and standard deviations are given in the subscript.}
\label{tab:main:results_roberta_base_rank=2}
\vspace{-5.0pt}
\end{table*}






\begin{wraptable}{Rt}{8.5cm}
\setlength{\tabcolsep}{3pt}
\centering
\caption{MLP results on MNIST with rank $r$ and $\alpha$ set to 1. In the case of \algname{AsymmLoRA} and \algname{RAC-LoRA}, only the zero-initialized matrix is trained.}
\begin{tabular}{lcccccc}
\toprule
\textbf{Method} & $\mathcal{D}_A$ & $\mathcal{D}_B$ & \textbf{Acc} & \textbf{Train Params}\\\midrule
\algname{FPFT} & - & - & 98.0 & 54,700  \\\midrule
\algname{LoRA} & Gaussian & Zero & 83.8 & 1K  \\
\algname{COLA} & Gaussian & Zero & 92.6 & 1K \\ \midrule

\algname{LoRA} & Zero & Gaussian & 87.0 & 1K \\ 
\algname{COLA}  & Zero & Gaussian & 96.2 & 1K \\
\midrule

\algname{AsymmLoRA}  & Gaussian & Zero & 62.3 & 133 \\ 
\algname{RAC-LoRA} & Gaussian & Zero & 92.0 & 133\\ 
\midrule
\algname{AsymmLoRA} & Zero & Gaussian & 81.6 & 912 \\
\algname{RAC-LoRA} & Zero & Gaussian & 96.1 & 912\\
\bottomrule
\end{tabular}
\label{tab:mnist_rank1}
\end{wraptable}

\subsubsection{Results of MLPs on MNIST}\label{sec:exp-mnist}

In this section, we seek to isolate the effect of the chaining procedure on generalisation performance by restricting the capacity of the low-rank adaptations. 
This ensures that a single adaptation is not sufficient to reach performance
comparable with \algname{FPFT}, allowing us to explore how chaining adaptations can bridge this gap.

\textbf{Methodology.}
We first pre-train a 3-layer MLP on the first five classes (digits 0-4) and then adapt the network using \algname{LoRA}-based methods for recognizing the remaining five unseen classes (digits 5-9). The model is evaluated solely on these unseen classes\footnote{The setup is inspired by \url{https://github.com/sunildkumar/lora_from_scratch/}.}. we used rank 1 for the adaptations and trained using the \algname{AdamW} optimiser~\citep{loshchilov2017decoupled} with $\beta$ parameters 0.9 and 0.999,
$\epsilon=\num{1e-8}$, a constant learning rate of \num{2e-4} and a training batch size 128.

\textbf{Discussion.}
Table~\ref{tab:mnist_rank1} shows results for MNIST with different ranks and initialization.
%
\algname{LoRA} reaches around 90\% of the accuracy of \algname{FPFT} leaving some margin for improvement when using the chains.
\algname{COLA} constructs a sequence of \algname{LoRA} modules, delivering significant accuracy
improvements over \algname{LoRA} due to the chaining procedure.
The chaining allows \algname{COLA} to capture richer features (at the cost of training more parameters). 
However, both \algname{LoRA} and \algname{COLA} lack rigorous convergence guarantees.
\algname{AsymmLoRA} has been shown empirically to approximate the performance of \algname{LoRA} \citep{sun2024improving} --- but again no convergence result is provided.
Our proposed method (\algname{RAC-LoRA}) enjoys significant accuracy improvements over \algname{AsymmLoRA}, again due to the chaining procedure. \algname{RAC-LoRA} leverages a diverse learning process across different \algname{LoRA} blocks, which intuitively allows the model to capture a broader range of features. Crucially, \algname{RAC-LoRA} comes with convergence guarantees (Theorems~\ref{thm:GD} and  \ref{thm:PL-GD}).
Finally, we note that each iteration of \algname{RAC-LoRA} requires training only one matrix per \algname{LoRA} block, while \algname{COLA} needs training two matrices. This reduction in trainable parameters may offer advantages in resource-constrained settings, such as Federated Learning, where minimizing communication costs is critical.



\section{Conclusion}
In this work, we introduced \algname{RAC-LoRA}, a framework for parameter-efficient fine-tuning that enables interpolation between low-rank adaptation and full parameter fine-tuning. Motivated by the convergence challenges of \algname{LoRA}, we propose the iterative algorithm \algname{RAC-LoRA} and provide convergence guarantees across various settings, including gradient descent, stochastic gradient descent, and random reshuffling.
We extended this framework to the
federated learning setup, where
\algname{RAC-LoRA} has advantages over
competing algorithms in terms of communication efficiency.
Finally, 
we validate our theoretical results empirically
in both convex problems, such as linear and logistic regression, and non-convex problems, such as MLPs and LLMs, finding that its chaining procedure is advantageous in settings where standard low-rank adaptation approaches (such as \algname{LoRA} and \algname{AsymmLoRA}) fail to capture the richness of a full-parameter fine-tuning.

\section*{Acknowledgements}

The research reported in this publication was supported by funding from King Abdullah University of Science and Technology (KAUST): i) KAUST Baseline Research Scheme, ii) Center of Excellence for Generative AI, under award number 5940, iii) SDAIA-KAUST Center of Excellence in Artificial Intelligence and Data Science.

\bibliography{iclr2025_conference}

\begin{thebibliography}{68}
\providecommand{\natexlab}[1]{#1}
\providecommand{\url}[1]{\texttt{#1}}
\expandafter\ifx\csname urlstyle\endcsname\relax
  \providecommand{\doi}[1]{doi: #1}\else
  \providecommand{\doi}{doi: \begingroup \urlstyle{rm}\Url}\fi

\bibitem[Aghajanyan et~al.(2020)Aghajanyan, Zettlemoyer, and Gupta]{aghajanyan2020intrinsic}
Armen Aghajanyan, Luke Zettlemoyer, and Sonal Gupta.
\newblock Intrinsic dimensionality explains the effectiveness of language model fine-tuning.
\newblock \emph{arXiv preprint arXiv:2012.13255}, 2020.

\bibitem[Ahn et~al.(2020)Ahn, Yun, and Sra]{ahn2020sgd}
Kwangjun Ahn, Chulhee Yun, and Suvrit Sra.
\newblock {SGD with shuffling: Optimal rates without component convexity and large epoch requirements}.
\newblock \emph{Advances in Neural Information Processing Systems}, 33:\penalty0 17526--17535, 2020.

\bibitem[Bertsekas(2011)]{bertsekas2011incremental}
Dimitri~P Bertsekas.
\newblock Incremental proximal methods for large scale convex optimization.
\newblock \emph{Mathematical Programming}, 129\penalty0 (2):\penalty0 163--195, 2011.

\bibitem[Biderman et~al.(2024)Biderman, Ortiz, Portes, Paul, Greengard, Jennings, King, Havens, Chiley, Frankle, et~al.]{biderman2024lora}
Dan Biderman, Jose~Gonzalez Ortiz, Jacob Portes, Mansheej Paul, Philip Greengard, Connor Jennings, Daniel King, Sam Havens, Vitaliy Chiley, Jonathan Frankle, et~al.
\newblock {LoRA learns less and forgets less}.
\newblock \emph{arXiv preprint arXiv:2405.09673}, 2024.

\bibitem[Cha et~al.(2023)Cha, Lee, and Yun]{cha2023tighter}
Jaeyoung Cha, Jaewook Lee, and Chulhee Yun.
\newblock Tighter lower bounds for shuffling sgd: Random permutations and beyond.
\newblock In \emph{International Conference on Machine Learning}, pp.\  3855--3912. PMLR, 2023.

\bibitem[Charles \& Kone{\v{c}}n{\'y}(2020)Charles and Kone{\v{c}}n{\'y}]{charles2020outsized}
Zachary Charles and Jakub Kone{\v{c}}n{\'y}.
\newblock On the outsized importance of learning rates in local update methods.
\newblock \emph{arXiv preprint arXiv:2007.00878}, 2020.

\bibitem[Cho et~al.(2023)Cho, Sharma, Joshi, Xu, Kale, and Zhang]{cho2023convergence}
Yae~Jee Cho, Pranay Sharma, Gauri Joshi, Zheng Xu, Satyen Kale, and Tong Zhang.
\newblock On the convergence of federated averaging with cyclic client participation.
\newblock In \emph{International Conference on Machine Learning}, pp.\  5677--5721. PMLR, 2023.

\bibitem[Church et~al.(2021)Church, Chen, and Ma]{church2021emerging}
Kenneth~Ward Church, Zeyu Chen, and Yanjun Ma.
\newblock Emerging trends: A gentle introduction to fine-tuning.
\newblock \emph{Natural Language Engineering}, 27\penalty0 (6):\penalty0 763--778, 2021.

\bibitem[Condat \& Richt{\'a}rik(2022)Condat and Richt{\'a}rik]{condat2022randprox}
Laurent Condat and Peter Richt{\'a}rik.
\newblock Randprox: Primal-dual optimization algorithms with randomized proximal updates.
\newblock \emph{arXiv preprint arXiv:2207.12891}, 2022.

\bibitem[Condat et~al.(2023)Condat, Agarsk{\`y}, Malinovsky, and Richt{\'a}rik]{condat2023tamuna}
Laurent Condat, Ivan Agarsk{\`y}, Grigory Malinovsky, and Peter Richt{\'a}rik.
\newblock Tamuna: Doubly accelerated federated learning with local training, compression, and partial participation.
\newblock \emph{arXiv preprint arXiv:2302.09832}, 2023.

\bibitem[Danilova et~al.(2022)Danilova, Dvurechensky, Gasnikov, Gorbunov, Guminov, Kamzolov, and Shibaev]{danilova2022recent}
Marina Danilova, Pavel Dvurechensky, Alexander Gasnikov, Eduard Gorbunov, Sergey Guminov, Dmitry Kamzolov, and Innokentiy Shibaev.
\newblock Recent theoretical advances in non-convex optimization.
\newblock In \emph{High-Dimensional Optimization and Probability: With a View Towards Data Science}, pp.\  79--163. Springer, 2022.

\bibitem[Demidovich et~al.(2023)Demidovich, Malinovsky, Sokolov, and Richt{\'a}rik]{demidovich2023guide}
Yury Demidovich, Grigory Malinovsky, Igor Sokolov, and Peter Richt{\'a}rik.
\newblock {A guide through the zoo of biased SGD}.
\newblock \emph{Advances in Neural Information Processing Systems}, 36:\penalty0 23158--23171, 2023.

\bibitem[Glasgow et~al.(2022)Glasgow, Yuan, and Ma]{glasgow2022sharp}
Margalit~R Glasgow, Honglin Yuan, and Tengyu Ma.
\newblock Sharp bounds for federated averaging (local sgd) and continuous perspective.
\newblock In \emph{International Conference on Artificial Intelligence and Statistics}, pp.\  9050--9090. PMLR, 2022.

\bibitem[Gorbunov et~al.(2021)Gorbunov, Hanzely, and Richt{\'a}rik]{gorbunov2021local}
Eduard Gorbunov, Filip Hanzely, and Peter Richt{\'a}rik.
\newblock Local {SGD}: {U}nified theory and new efficient methods.
\newblock In \emph{International Conference on Artificial Intelligence and Statistics}, pp.\  3556--3564. PMLR, 2021.

\bibitem[Gower et~al.(2019)Gower, Loizou, Qian, Sailanbayev, Shulgin, and Richt{\'a}rik]{gower2019sgd}
Robert~Mansel Gower, Nicolas Loizou, Xun Qian, Alibek Sailanbayev, Egor Shulgin, and Peter Richt{\'a}rik.
\newblock {SGD: General analysis and improved rates}.
\newblock In \emph{International Conference on Machine Learning}, pp.\  5200--5209. PMLR, 2019.

\bibitem[Grudzie{\'n} et~al.(2023{\natexlab{a}})Grudzie{\'n}, Malinovsky, and Richt{\'a}rik]{grudzien2023can}
Micha{\l} Grudzie{\'n}, Grigory Malinovsky, and Peter Richt{\'a}rik.
\newblock Can 5th generation local training methods support client sampling? yes!
\newblock In \emph{International Conference on Artificial Intelligence and Statistics}, pp.\  1055--1092. PMLR, 2023{\natexlab{a}}.

\bibitem[Grudzie{\'n} et~al.(2023{\natexlab{b}})Grudzie{\'n}, Malinovsky, and Richt{\'a}rik]{grudzien2023improving}
Micha{\l} Grudzie{\'n}, Grigory Malinovsky, and Peter Richt{\'a}rik.
\newblock Improving accelerated federated learning with compression and importance sampling.
\newblock \emph{arXiv preprint arXiv:2306.03240}, 2023{\natexlab{b}}.

\bibitem[Han et~al.(2024)Han, Gao, Liu, Zhang, et~al.]{han2024parameter}
Zeyu Han, Chao Gao, Jinyang Liu, Sai~Qian Zhang, et~al.
\newblock Parameter-efficient fine-tuning for large models: A comprehensive survey.
\newblock \emph{arXiv preprint arXiv:2403.14608}, 2024.

\bibitem[He et~al.(2021)He, Zhou, Ma, Berg-Kirkpatrick, and Neubig]{he2021towards}
Junxian He, Chunting Zhou, Xuezhe Ma, Taylor Berg-Kirkpatrick, and Graham Neubig.
\newblock Towards a unified view of parameter-efficient transfer learning.
\newblock \emph{arXiv preprint arXiv:2110.04366}, 2021.

\bibitem[Horv{\'a}th et~al.(2022)Horv{\'a}th, Sanjabi, Xiao, Richt{\'a}rik, and Rabbat]{horvath2022fedshuffle}
Samuel Horv{\'a}th, Maziar Sanjabi, Lin Xiao, Peter Richt{\'a}rik, and Michael Rabbat.
\newblock {FedShuffle: Recipes for better use of local work in federated learning}.
\newblock \emph{arXiv preprint arXiv:2204.13169}, 2022.

\bibitem[Hu et~al.(2021)Hu, Shen, Wallis, Allen-Zhu, Li, Wang, Wang, and Chen]{hu2021lora}
Edward~J Hu, Yelong Shen, Phillip Wallis, Zeyuan Allen-Zhu, Yuanzhi Li, Shean Wang, Lu~Wang, and Weizhu Chen.
\newblock {LoRA: Low-rank adaptation of large language models}.
\newblock \emph{arXiv preprint arXiv:2106.09685}, 2021.

\bibitem[Jain et~al.(2019)Jain, Nagaraj, and Netrapalli]{jain2019sgd}
Prateek Jain, Dheeraj Nagaraj, and Praneeth Netrapalli.
\newblock {SGD without replacement: Sharper rates for general smooth convex functions}.
\newblock \emph{arXiv preprint arXiv:1903.01463}, 2019.

\bibitem[Kairouz et~al.(2021)Kairouz, McMahan, Avent, Bellet, Bennis, Bhagoji, Bonawitz, Charles, Cormode, Cummings, et~al.]{kairouz2021advances}
Peter Kairouz, H~Brendan McMahan, Brendan Avent, Aur{\'e}lien Bellet, Mehdi Bennis, Arjun~Nitin Bhagoji, Kallista Bonawitz, Zachary Charles, Graham Cormode, Rachel Cummings, et~al.
\newblock Advances and open problems in federated learning.
\newblock \emph{Foundations and trends{\textregistered} in machine learning}, 14\penalty0 (1--2):\penalty0 1--210, 2021.

\bibitem[Karimireddy et~al.(2020)Karimireddy, Kale, Mohri, Reddi, Stich, and Suresh]{karimireddy2020scaffold}
Sai~Praneeth Karimireddy, Satyen Kale, Mehryar Mohri, Sashank Reddi, Sebastian Stich, and Ananda~Theertha Suresh.
\newblock Scaffold: Stochastic controlled averaging for federated learning.
\newblock In \emph{International conference on machine learning}, pp.\  5132--5143. PMLR, 2020.

\bibitem[Khaled \& Richt{\'a}rik(2022)Khaled and Richt{\'a}rik]{khaled2020better}
Ahmed Khaled and Peter Richt{\'a}rik.
\newblock Better theory for {SGD} in the nonconvex world.
\newblock \emph{Transactions on Machine Learning Research}, 2022.

\bibitem[Khaled et~al.(2019)Khaled, Mishchenko, and Richt{\'a}rik]{khaled2019first}
Ahmed Khaled, Konstantin Mishchenko, and Peter Richt{\'a}rik.
\newblock First analysis of local {GD} on heterogeneous data.
\newblock \emph{arXiv preprint arXiv:1909.04715}, 2019.

\bibitem[Khaled et~al.(2020)Khaled, Mishchenko, and Richt{\'a}rik]{khaled2020tighter}
Ahmed Khaled, Konstantin Mishchenko, and Peter Richt{\'a}rik.
\newblock Tighter theory for local {SGD} on identical and heterogeneous data.
\newblock In \emph{International Conference on Artificial Intelligence and Statistics}, pp.\  4519--4529. PMLR, 2020.

\bibitem[Khaled et~al.(2023)Khaled, Sebbouh, Loizou, Gower, and Richt{\'a}rik]{khaled2023unified}
Ahmed Khaled, Othmane Sebbouh, Nicolas Loizou, Robert~M Gower, and Peter Richt{\'a}rik.
\newblock Unified analysis of stochastic gradient methods for composite convex and smooth optimization.
\newblock \emph{Journal of Optimization Theory and Applications}, 199\penalty0 (2):\penalty0 499--540, 2023.

\bibitem[Khodak et~al.(2021)Khodak, Tu, Li, Li, Balcan, Smith, and Talwalkar]{khodak2021federated}
Mikhail Khodak, Renbo Tu, Tian Li, Liam Li, Maria-Florina~F Balcan, Virginia Smith, and Ameet Talwalkar.
\newblock Federated hyperparameter tuning: Challenges, baselines, and connections to weight-sharing.
\newblock \emph{Advances in Neural Information Processing Systems}, 34:\penalty0 19184--19197, 2021.

\bibitem[Koloskova et~al.(2020)Koloskova, Loizou, Boreiri, Jaggi, and Stich]{koloskova2020unified}
Anastasia Koloskova, Nicolas Loizou, Sadra Boreiri, Martin Jaggi, and Sebastian Stich.
\newblock A unified theory of decentralized sgd with changing topology and local updates.
\newblock In \emph{International Conference on Machine Learning}, pp.\  5381--5393. PMLR, 2020.

\bibitem[Kone\v{c}n\'{y} et~al.(2016)Kone\v{c}n\'{y}, McMahan, Yu, Richt\'{a}rik, Suresh, and Bacon]{konecnyFL}
Jakub Kone\v{c}n\'{y}, H.~Brendan McMahan, Felix~X. Yu, Peter Richt\'{a}rik, Ananda~Theertha Suresh, and Dave Bacon.
\newblock Federated learning: Strategies for improving communication efficiency.
\newblock \emph{arXiv preprint arXiv:1610.05492}, 2016.

\bibitem[Kuang et~al.(2024)Kuang, Qian, Li, Chen, Gao, Pan, Xie, Li, Ding, and Zhou]{kuang2024federatedscope}
Weirui Kuang, Bingchen Qian, Zitao Li, Daoyuan Chen, Dawei Gao, Xuchen Pan, Yuexiang Xie, Yaliang Li, Bolin Ding, and Jingren Zhou.
\newblock {FederatedScope-LLM: A comprehensive package for fine-tuning large language models in federated learning}.
\newblock In \emph{Proceedings of the 30th ACM SIGKDD Conference on Knowledge Discovery and Data Mining}, pp.\  5260--5271, 2024.

\bibitem[Li et~al.(2018)Li, Farkhoor, Liu, and Yosinski]{li2018measuring}
Chunyuan Li, Heerad Farkhoor, Rosanne Liu, and Jason Yosinski.
\newblock Measuring the intrinsic dimension of objective landscapes.
\newblock \emph{arXiv preprint arXiv:1804.08838}, 2018.

\bibitem[Liu(2019)]{liu2019roberta}
Yinhan Liu.
\newblock {RoBERTa: A robustly optimized BERT pretraining approach}.
\newblock \emph{arXiv preprint arXiv:1907.11692}, 2019.

\bibitem[Loshchilov(2017)]{loshchilov2017decoupled}
I~Loshchilov.
\newblock Decoupled weight decay regularization.
\newblock \emph{arXiv preprint arXiv:1711.05101}, 2017.

\bibitem[Lu et~al.(2022)Lu, Meng, and De~Sa]{lu2022general}
Yucheng Lu, Si~Yi Meng, and Christopher De~Sa.
\newblock A general analysis of example-selection for stochastic gradient descent.
\newblock In \emph{International Conference on Learning Representations (ICLR)}, volume~10, 2022.

\bibitem[Malinovskiy et~al.(2020)Malinovskiy, Kovalev, Gasanov, Condat, and Richt\'{a}rik]{malinovskiy2020local}
Grigory Malinovskiy, Dmitry Kovalev, Elnur Gasanov, Laurent Condat, and Peter Richt\'{a}rik.
\newblock From local {SGD} to local fixed-point methods for federated learning.
\newblock In \emph{International Conference on Machine Learning}, pp.\  6692--6701. PMLR, 2020.

\bibitem[Malinovsky \& Richt{\'a}rik(2022)Malinovsky and Richt{\'a}rik]{malinovsky2022federated}
Grigory Malinovsky and Peter Richt{\'a}rik.
\newblock Federated random reshuffling with compression and variance reduction.
\newblock \emph{arXiv preprint arXiv:2205.03914}, 2022.

\bibitem[Malinovsky et~al.(2022)Malinovsky, Yi, and Richt{\'a}rik]{malinovsky2022variance}
Grigory Malinovsky, Kai Yi, and Peter Richt{\'a}rik.
\newblock Variance reduced proxskip: Algorithm, theory and application to federated learning.
\newblock \emph{Advances in Neural Information Processing Systems}, 35:\penalty0 15176--15189, 2022.

\bibitem[Malinovsky et~al.(2023{\natexlab{a}})Malinovsky, Horv{\'a}th, Burlachenko, and Richt{\'a}rik]{malinovsky2023federated}
Grigory Malinovsky, Samuel Horv{\'a}th, Konstantin Burlachenko, and Peter Richt{\'a}rik.
\newblock Federated learning with regularized client participation.
\newblock \emph{arXiv preprint arXiv:2302.03662}, 2023{\natexlab{a}}.

\bibitem[Malinovsky et~al.(2023{\natexlab{b}})Malinovsky, Mishchenko, and Richt{\'a}rik]{malinovsky2023server}
Grigory Malinovsky, Konstantin Mishchenko, and Peter Richt{\'a}rik.
\newblock Server-side stepsizes and sampling without replacement provably help in federated optimization.
\newblock In \emph{Proceedings of the 4th International Workshop on Distributed Machine Learning}, pp.\  85--104, 2023{\natexlab{b}}.

\bibitem[Malinovsky et~al.(2023{\natexlab{c}})Malinovsky, Sailanbayev, and Richt{\'a}rik]{malinovsky2023random}
Grigory Malinovsky, Alibek Sailanbayev, and Peter Richt{\'a}rik.
\newblock Random reshuffling with variance reduction: New analysis and better rates.
\newblock In \emph{Uncertainty in Artificial Intelligence}, pp.\  1347--1357. PMLR, 2023{\natexlab{c}}.

\bibitem[Maranjyan et~al.(2022)Maranjyan, Safaryan, and Richt{\'a}rik]{maranjyan2022gradskip}
Artavazd Maranjyan, Mher Safaryan, and Peter Richt{\'a}rik.
\newblock Gradskip: Communication-accelerated local gradient methods with better computational complexity.
\newblock \emph{arXiv preprint arXiv:2210.16402}, 2022.

\bibitem[Mishchenko et~al.(2020)Mishchenko, Khaled, and Richt{\'a}rik]{mishchenko2020random}
Konstantin Mishchenko, Ahmed Khaled, and Peter Richt{\'a}rik.
\newblock Random reshuffling: Simple analysis with vast improvements.
\newblock \emph{Advances in Neural Information Processing Systems}, 33:\penalty0 17309--17320, 2020.

\bibitem[Mishchenko et~al.(2022{\natexlab{a}})Mishchenko, Khaled, and Richt{\'a}rik]{mishchenko2022proximal}
Konstantin Mishchenko, Ahmed Khaled, and Peter Richt{\'a}rik.
\newblock Proximal and federated random reshuffling.
\newblock In \emph{International Conference on Machine Learning}, pp.\  15718--15749. PMLR, 2022{\natexlab{a}}.

\bibitem[Mishchenko et~al.(2022{\natexlab{b}})Mishchenko, Malinovsky, Stich, and Richt{\'a}rik]{mishchenko2022proxskip}
Konstantin Mishchenko, Grigory Malinovsky, Sebastian Stich, and Peter Richt{\'a}rik.
\newblock Proxskip: Yes! local gradient steps provably lead to communication acceleration! finally!
\newblock In \emph{International Conference on Machine Learning}, pp.\  15750--15769. PMLR, 2022{\natexlab{b}}.

\bibitem[Nedi{\'c} \& Bertsekas(2001)Nedi{\'c} and Bertsekas]{nedic2001convergence}
Angelia Nedi{\'c} and Dimitri Bertsekas.
\newblock Convergence rate of incremental subgradient algorithms.
\newblock \emph{Stochastic optimization: algorithms and applications}, pp.\  223--264, 2001.

\bibitem[Nesterov(2004)]{nesterov2004introductory}
Yu~E Nesterov.
\newblock Introductory lectures on convex optimization. a basic course.
\newblock 2004.

\bibitem[Nguyen et~al.(2021)Nguyen, Tran-Dinh, Phan, Nguyen, and Van~Dijk]{nguyen2021unified}
Lam~M Nguyen, Quoc Tran-Dinh, Dzung~T Phan, Phuong~Ha Nguyen, and Marten Van~Dijk.
\newblock A unified convergence analysis for shuffling-type gradient methods.
\newblock \emph{Journal of Machine Learning Research}, 22\penalty0 (207):\penalty0 1--44, 2021.

\bibitem[Recht \& R{\'e}(2012)Recht and R{\'e}]{recht2012toward}
Benjamin Recht and Christopher R{\'e}.
\newblock Toward a noncommutative arithmetic-geometric mean inequality: Conjectures, case-studies, and consequences.
\newblock In \emph{Conference on Learning Theory}, pp.\  11--1. JMLR Workshop and Conference Proceedings, 2012.

\bibitem[Richt\'{a}rik \& Tak\'{a}\v{c}(2016)Richt\'{a}rik and Tak\'{a}\v{c}]{PCDM}
Peter Richt\'{a}rik and Martin Tak\'{a}\v{c}.
\newblock Parallel coordinate descent methods for big data optimization.
\newblock \emph{Mathematical Programming}, 156\penalty0 (1-2):\penalty0 433--484, 2016.

\bibitem[Ruder(2016)]{ruder2016overview}
Sebastian Ruder.
\newblock An overview of gradient descent optimization algorithms.
\newblock \emph{arXiv preprint arXiv:1609.04747}, 2016.

\bibitem[Sadiev et~al.(2022{\natexlab{a}})Sadiev, Kovalev, and Richt{\'a}rik]{sadiev2022communication}
Abdurakhmon Sadiev, Dmitry Kovalev, and Peter Richt{\'a}rik.
\newblock Communication acceleration of local gradient methods via an accelerated primal-dual algorithm with an inexact prox.
\newblock \emph{Advances in Neural Information Processing Systems}, 35:\penalty0 21777--21791, 2022{\natexlab{a}}.

\bibitem[Sadiev et~al.(2022{\natexlab{b}})Sadiev, Malinovsky, Gorbunov, Sokolov, Khaled, Burlachenko, and Richt{\'a}rik]{sadiev2022federated}
Abdurakhmon Sadiev, Grigory Malinovsky, Eduard Gorbunov, Igor Sokolov, Ahmed Khaled, Konstantin Burlachenko, and Peter Richt{\'a}rik.
\newblock Federated optimization algorithms with random reshuffling and gradient compression.
\newblock \emph{arXiv preprint arXiv:2206.07021}, 2022{\natexlab{b}}.

\bibitem[Safran \& Shamir(2021)Safran and Shamir]{safran2021random}
Itay Safran and Ohad Shamir.
\newblock {Random shuffling beats SGD only after many epochs on ill-conditioned problems}.
\newblock \emph{Advances in Neural Information Processing Systems}, 34:\penalty0 15151--15161, 2021.

\bibitem[Shapiro \& Wardi(1996)Shapiro and Wardi]{shapiro1996convergence}
A~Shapiro and Y~Wardi.
\newblock Convergence analysis of gradient descent stochastic algorithms.
\newblock \emph{Journal of Optimization Theory and Applications}, 91:\penalty0 439--454, 1996.

\bibitem[Sun(2020)]{sun2020optimization}
Ruo-Yu Sun.
\newblock Optimization for deep learning: An overview.
\newblock \emph{Journal of the Operations Research Society of China}, 8\penalty0 (2):\penalty0 249--294, 2020.

\bibitem[Sun et~al.(2024)Sun, Li, Li, and Ding]{sun2024improving}
Youbang Sun, Zitao Li, Yaliang Li, and Bolin Ding.
\newblock {Improving LoRA in privacy-preserving federated learning}.
\newblock \emph{arXiv preprint arXiv:2403.12313}, 2024.

\bibitem[Sun et~al.(2022)Sun, Yang, Liu, Yin, Li, and Xu]{sun2022recent}
Zehua Sun, Huanqi Yang, Kai Liu, Zhimeng Yin, Zhenjiang Li, and Weitao Xu.
\newblock {Recent advances in LoRa: A comprehensive survey}.
\newblock \emph{ACM Transactions on Sensor Networks}, 18\penalty0 (4):\penalty0 1--44, 2022.

\bibitem[Vrban{\v{c}}i{\v{c}} \& Podgorelec(2020)Vrban{\v{c}}i{\v{c}} and Podgorelec]{vrbanvcivc2020transfer}
Grega Vrban{\v{c}}i{\v{c}} and Vili Podgorelec.
\newblock Transfer learning with adaptive fine-tuning.
\newblock \emph{IEEE Access}, 8:\penalty0 196197--196211, 2020.

\bibitem[Wang(2018)]{wang2018glue}
Alex Wang.
\newblock Glue: A multi-task benchmark and analysis platform for natural language understanding.
\newblock \emph{arXiv preprint arXiv:1804.07461}, 2018.

\bibitem[Woodworth et~al.(2020{\natexlab{a}})Woodworth, Patel, Stich, Dai, Bullins, Mcmahan, Shamir, and Srebro]{woodworth2020local}
Blake Woodworth, Kumar~Kshitij Patel, Sebastian Stich, Zhen Dai, Brian Bullins, Brendan Mcmahan, Ohad Shamir, and Nathan Srebro.
\newblock Is local {SGD} better than minibatch {SGD}?
\newblock In \emph{International Conference on Machine Learning}, pp.\  10334--10343. PMLR, 2020{\natexlab{a}}.

\bibitem[Woodworth et~al.(2020{\natexlab{b}})Woodworth, Patel, and Srebro]{woodworth2020minibatch}
Blake~E Woodworth, Kumar~Kshitij Patel, and Nati Srebro.
\newblock Minibatch vs local sgd for heterogeneous distributed learning.
\newblock \emph{Advances in Neural Information Processing Systems}, 33:\penalty0 6281--6292, 2020{\natexlab{b}}.

\bibitem[Xia et~al.(2024)Xia, Qin, and Hazan]{xia2024chain}
Wenhan Xia, Chengwei Qin, and Elad Hazan.
\newblock {Chain of LoRA: Efficient fine-tuning of language models via residual learning}.
\newblock \emph{arXiv preprint arXiv:2401.04151}, 2024.

\bibitem[Xu et~al.(2023)Xu, Xie, Qin, Tao, and Wang]{xu2023parameter}
Lingling Xu, Haoran Xie, Si-Zhao~Joe Qin, Xiaohui Tao, and Fu~Lee Wang.
\newblock Parameter-efficient fine-tuning methods for pretrained language models: A critical review and assessment.
\newblock \emph{arXiv preprint arXiv:2312.12148}, 2023.

\bibitem[Yun et~al.(2021)Yun, Rajput, and Sra]{yun2021minibatch}
Chulhee Yun, Shashank Rajput, and Suvrit Sra.
\newblock Minibatch vs local sgd with shuffling: Tight convergence bounds and beyond.
\newblock \emph{arXiv preprint arXiv:2110.10342}, 2021.

\bibitem[Zhou(2018)]{zhou2018fenchel}
Xingyu Zhou.
\newblock On the fenchel duality between strong convexity and lipschitz continuous gradient.
\newblock \emph{arXiv preprint arXiv:1803.06573}, 2018.

\bibitem[Zhu et~al.(2024)Zhu, Greenewald, Nadjahi, Borde, Gabrielsson, Choshen, Ghassemi, Yurochkin, and Solomon]{zhu2024asymmetry}
Jiacheng Zhu, Kristjan Greenewald, Kimia Nadjahi, Haitz S{\'a}ez de~Oc{\'a}riz Borde, Rickard~Br{\"u}el Gabrielsson, Leshem Choshen, Marzyeh Ghassemi, Mikhail Yurochkin, and Justin Solomon.
\newblock Asymmetry in low-rank adapters of foundation models.
\newblock \emph{arXiv preprint arXiv:2402.16842}, 2024.

\end{thebibliography}
\bibliographystyle{iclr2025_conference}

\newpage
\appendix
\part*{Appendix}

\section{Results on Convex Optimization Problems}
\label{sec:app:exp-convex}
\subsection{Logistic Regression}
We performed our analysis in a controlled environment using logistic regression with quadratic regularization on synthetic data. In this configuration, we set \(d = 100\), employed weight matrices of size \(10 \times 10\), and used 2,000 samples, with the regularization term fixed at \(0.1\). As shown in Figure~\ref{fig:logreg}, the method demonstrates convergence across different ranks, and when the rank is set to full rank, we observe convergence that mirrors that of \algname{FPFT}.

\begin{figure}[h]
\centering
\includegraphics[width=0.7\textwidth]{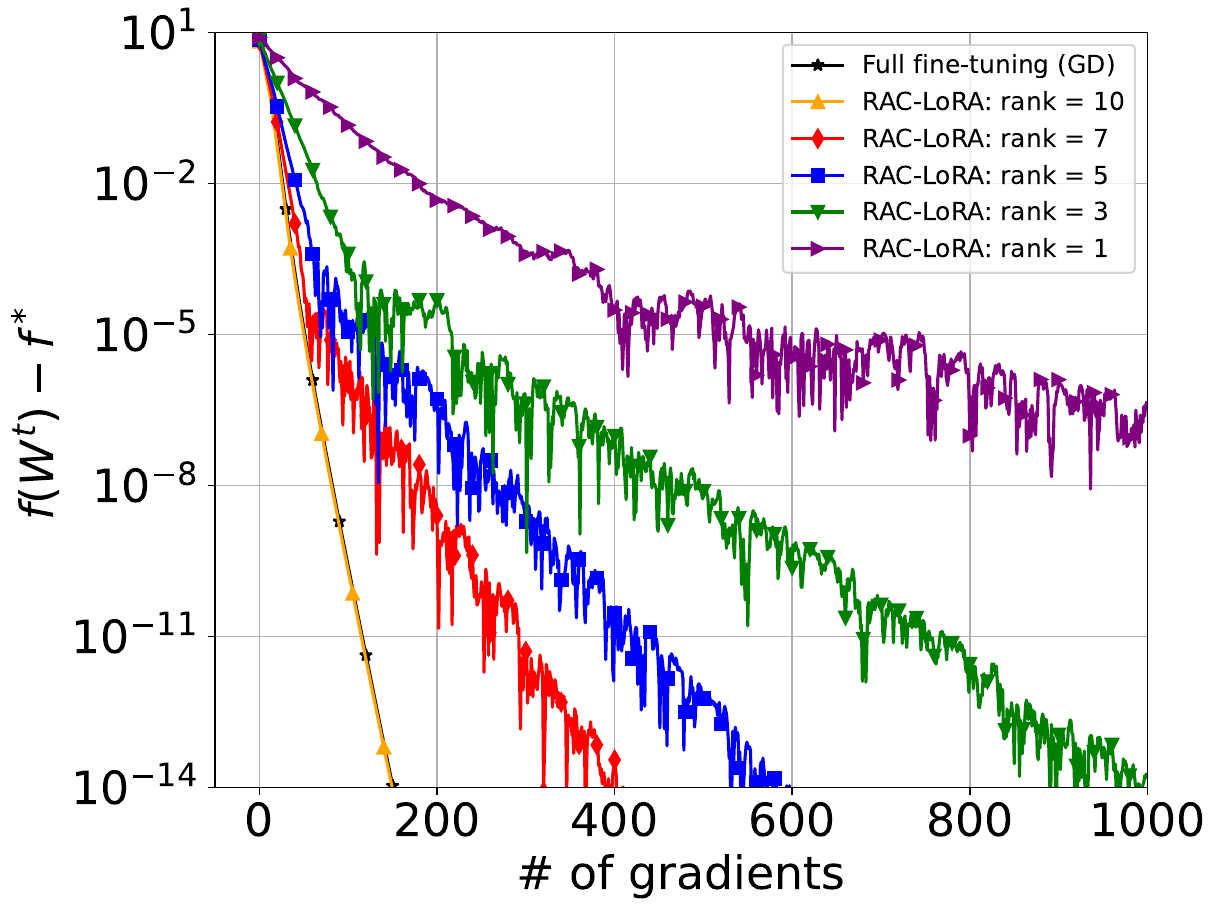}

\caption{\algnamesmall{RAC-LoRA} convergence with varying ranks and step sizes on a logistic regression problem.}\label{fig:logreg}
\end{figure}

\section{Results on Non-Convex Optimization Problems}
\label{sec:app:nonconvex}

\subsection{Additional Results of RoBERTa on NLP Tasks}

Table~\ref{results_roberta_base_rank=2} reports additional configurations of the number of epochs per chain and the number of chains on the GLUE benchmark. These results further corroborate the discussion in Section~\ref{sec:exp-nonconvex}.

\begin{table*}[htb!]
\resizebox{\textwidth}{!}{%
\centering
{
\begin{tabular}{lccccccc}
 \toprule
{{Method}} & {{\# Chains}} & {{\# Epochs}} & \textcolor{black}{{MRPC}} & \textcolor{black}{{CoLA}} & \textcolor{black}{{RTE}} & \textcolor{black}{{STS-B}} &  {Avg}\\

 \midrule

\algname{FPFT}* &  \multirow{2}{*}{1} & \multirow{2}{*}{30, 80, 80, 40} & 
90.2\textsubscript{$\pm$0.0} &
63.6\textsubscript{$\pm$0.0} &
78.7\textsubscript{$\pm$0.0} & 
91.2\textsubscript{$\pm$0.0} & 
80.9\\

\algname{LoRA}* & & & 
89.7\textsubscript{$\pm$0.7} &
63.4\textsubscript{$\pm$1.2} &
86.6\textsubscript{$\pm$0.7} & 
91.5\textsubscript{$\pm$0.2} & 
82.8\\

\midrule

\algname{LoRA} & \multirow{2}{*}{1} & \multirow{2}{*}{20} & 
86.8\textsubscript{$\pm$0.8} &
58.0\textsubscript{$\pm$0.4} &
{71.4\textsubscript{$\pm$0.7}} & 
{90.3\textsubscript{$\pm$0.1}} & 
76.6\\
 
\algname{AsymmLoRA} & & & 
85.5\textsubscript{$\pm$0.5} &
56.5\textsubscript{$\pm$1.5} &
69.2\textsubscript{$\pm$0.2} & 
89.6\textsubscript{$\pm$0.1} & 
75.2\\

\cdashline{1-8}
\multirow{2}{*}{\algname{COLA}} & 2 & 10 & 
{87.1\textsubscript{$\pm$0.2}} &
{58.4\textsubscript{$\pm$1.5}} &
69.9\textsubscript{$\pm$0.9} &
90.3\textsubscript{$\pm$0.2} & 
76.4

\\
& 10 & 2 & 
84.2\textsubscript{$\pm$1.1} &
54.2\textsubscript{$\pm$0.4}&
64.6\textsubscript{$\pm$1.3}& 
89.1\textsubscript{$\pm$0.1} & 
73.0\\

\cdashline{1-8}
\multirow{2}{*}{\algname{RAC-LoRA}} & 2 & 10
&
85.6\textsubscript{$\pm$1.7} &
55.3\textsubscript{$\pm$1.2} &
68.6\textsubscript{$\pm$1.0} & 
89.4\textsubscript{$\pm$0.2} & 
74.7\\

& 10 & 2
&
85.4\textsubscript{$\pm$0.4} &
55.1\textsubscript{$\pm$1.2} &
65.5\textsubscript{$\pm$0.9} & 
89.3\textsubscript{$\pm$0.1} & 
73.8\\

\\[-2.5ex]
\midrule

\algname{LoRA} & \multirow{2}{*}{1} & \multirow{2}{*}{50}
&
{88.2\textsubscript{$\pm$0.3}} &
{60.1\textsubscript{$\pm$0.4}} &
{74.4\textsubscript{$\pm$0.9}} & 
{90.6\textsubscript{$\pm$0.1}} & 
78.3\\
 
\algname{AsymmLoRA} &  & 
&
86.4\textsubscript{$\pm$1.0} &
57.4\textsubscript{$\pm$0.3} &
69.9\textsubscript{$\pm$1.8} & 
90.3\textsubscript{$\pm$0.1} & 
76.0\\

\cdashline{1-8}
\multirow{2}{*}{\algname{COLA}} & 5 & 10
& 
87.8\textsubscript{$\pm$1.1} &
59.3\textsubscript{$\pm$2.1} &
71.2\textsubscript{$\pm$1.2} &
90.6\textsubscript{$\pm$0.2} & 
77.2\\
& 10 & 5 
&
87.7\textsubscript{$\pm$0.5} &
58.1\textsubscript{$\pm$1.2} &
70.9\textsubscript{$\pm$0.5} &
90.2\textsubscript{$\pm$0.2} & 
76.7\\

\cdashline{1-8}
\multirow{2}{*}{\algname{RAC-LoRA}} & 5 & 10
& 
87.2\textsubscript{$\pm$0.6} &
57.6\textsubscript{$\pm$0.5} &
70.6\textsubscript{$\pm$0.7} &
90.2\textsubscript{$\pm$0.1} & 
76.4\\
& 10 & 5
&
87.5\textsubscript{$\pm$0.4} &
57.8\textsubscript{$\pm$1.0} &
70.3\textsubscript{$\pm$1.2} &
90.2\textsubscript{$\pm$0.2} & 
76.5\\

\midrule

\algname{LoRA} & \multirow{2}{*}{1} & \multirow{2}{*}{100}  
&
87.7\textsubscript{$\pm$0.2} &
{60.8\textsubscript{$\pm$0.2}} &
{75.2\textsubscript{$\pm$1.5}} & 
90.2\textsubscript{$\pm$0.1} & 
78.5\\
 
\algname{AsymmLoRA} & & 
&
86.9\textsubscript{$\pm$0.3} &
58.7\textsubscript{$\pm$1.0} &
71.0\textsubscript{$\pm$3.3} & 
90.4\textsubscript{$\pm$0.0} & 
76.8\\

\algname{COLA} & 10 & 10
& 
{88.0\textsubscript{$\pm$0.8}} &
59.5\textsubscript{$\pm$1.0} &
72.1\textsubscript{$\pm$0.9}&
{90.7\textsubscript{$\pm$0.2}} & 
77.6\\

\algname{RAC-LoRA} & 10 & 10
& 
87.0\textsubscript{$\pm$0.7} &
58.5\textsubscript{$\pm$0.1} &
72.3\textsubscript{$\pm$1.5}&
90.3\textsubscript{$\pm$0.0} & 
77.0\\

\bottomrule
\end{tabular}
}
}

\caption{{Performance of the methods using RoBERTa-base for rank 2}. The experiments are based on 4 tasks from the GLUE benchmark. * denotes the results reported in \cite{hu2021lora}. We report Matthews correlation coefficient for the CoLA dataset, Pearson correlation coefficient for STS-B, and accuracy for the remaining tasks, with the standard deviations given in the subscript. The results are obtained using 3 random seeds.}
\label{results_roberta_base_rank=2}
\vspace{-5.0pt}
\end{table*}

\subsection{Ablation on number of epochs per block in the chains}

Convergence proof for \algname{RAC-LoRA} (Corollary \ref{cor-RR-gen} and Corollary \ref{cor-RR-PL}) states that each \algname{LoRA} module shall be optimized for one epoch only. However, good approximations can also be obtained using more epochs per block and hence fewer blocks (i.e., fewer parameters), as we show in Table~\ref{tab:varying-epochs} for the case of MLP on MNIST. 

Similarly, we plot the training loss curves for RoBERTa-base on the RTE dataset in Figure~\ref{fig:train_loss_rte}. We observe that all setups reach the same value at convergence with similar speed.



\begin{table}[ht]
\centering
\begin{tabular}{ccccccc}
\toprule
& \multicolumn{6}{c}{\textbf{Number of epochs per block}} \\\cmidrule(lr){2-7}
& \textbf{1} & \textbf{2} & \textbf{3} & \textbf{4} & \textbf{5} & \textbf{10} \\
\midrule

\algname{COLA} & 96.2 & 95.8 & 95.9 & 95.1 & 95.4 & 94.5 \\

\algname{RAC-LoRA} & 96.1 & 95.6 & 95.6 & 94.9 & 94.7 & 93.9 \\

\bottomrule
\end{tabular}
\caption{Accuracy at varying epochs for each block in the chained methods (\algname{COLA} and \algname{RAC-LoRA}). The setup is the same as in Table~\ref{tab:mnist_rank1}, with a zero-initialized $A$ matrix and a Gaussian-initialized $B$ matrix. To ensure a fair comparison, the product of the number of epochs per block and the number of blocks is kept constant at 50. The number of trainable parameters for \algname{COLA} and \algname{RAC-LoRA} are 1K and 912, respectively.}
\label{tab:varying-epochs}
\end{table}

\begin{figure}
    \centering
    \includegraphics[width=0.8\linewidth]{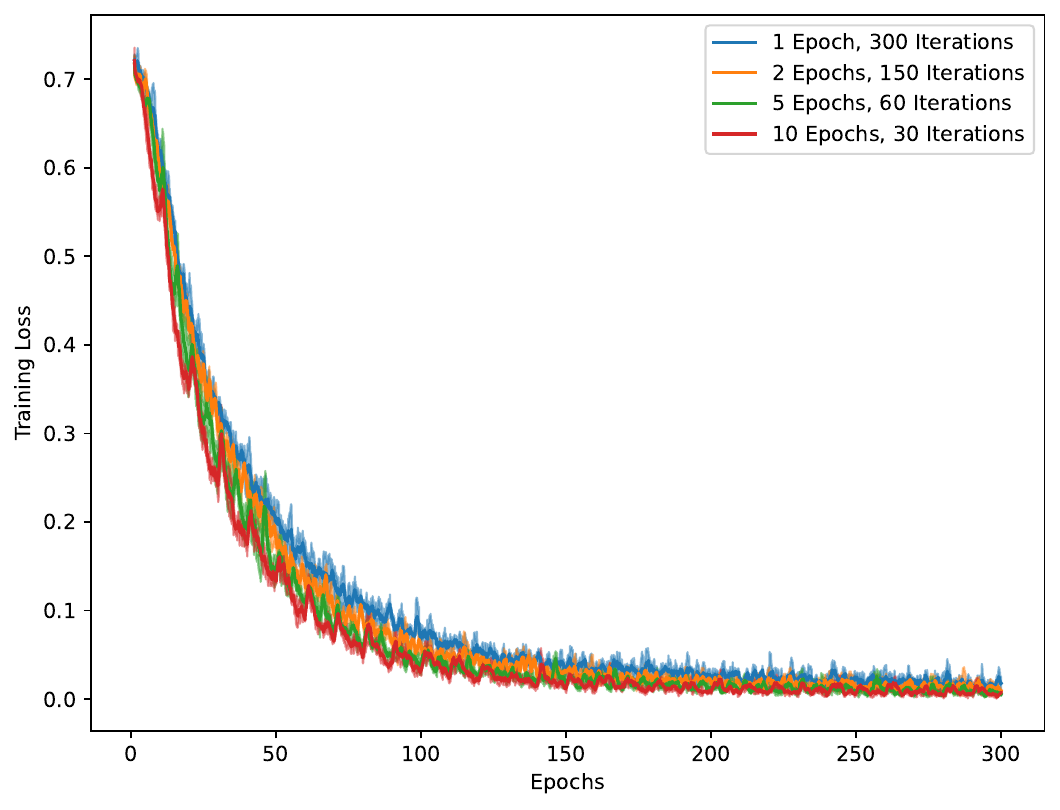}
    \caption{\algname{RAC-LoRA} training loss curves at a fixed computational budget for varying epochs for each block in the chain. RoBERTa-base with rank 2.}
    \label{fig:train_loss_rte}
\end{figure}




\section{Analysis of RAC-LoRA with Gradient Descent} \label{sec:GD-proofs}
\subsection{Proof of Theorem \ref{thm:GD}}
\label{sec:GD-pr}


The proof is provided for Left Sketch (Definition \ref{def:left}). The result for Right Sketch (Definition \ref{def:right}) can be derived by following the same steps.
\begin{proof}
We begin by examining the implications of Assumption~\ref{asm:L-smooth}. The relationships between various conditions associated with Assumption~\ref{asm:L-smooth} are discussed in detail in \citet{nesterov2004introductory}.
\begin{align*}
    f(W^{t+1}) \leq& f(W^t) + \left\langle \nabla f(W^t), W^{t+1} - W^t \right \rangle + \frac{L}{2}\left\| W^{t+1} - W^t \right\|^2
\end{align*}
Using the update rule $W^{t+1} = W^t - \gamma H^t_B \nabla f(W^t)$ we get
\begin{align*}
   f(W^{t+1})      \leq& f(W^t) + \left\langle \nabla f(W^t), - \gamma H^t_B \nabla f(W^t) \right \rangle + \frac{L}{2}\left\|  - \gamma H^t_B \nabla f(W^t) \right\|^2\\
   \leq& f(W^t) -\gamma \left\langle \nabla f(W^t),  H^t_B \nabla f(W^t) \right \rangle + \frac{L}{2} \gamma^2 \left\|    H^t_B \nabla f(W^t) \right\|^2\\
      \leq& f(W^t) -\gamma \left\langle \nabla f(W^t),  H^t_B \nabla f(W^t) \right \rangle + \frac{L}{2} \gamma^2 \left\langle H^t_B \nabla f(W^t),  H^t_B \nabla f(W^t) \right \rangle\\
            \leq& f(W^t) -\gamma \left\langle \nabla f(W^t),  H^t_B \nabla f(W^t) \right \rangle + \frac{L}{2} \gamma^2 \left\langle  \nabla f(W^t),  (H^t_B)^\top H^t_B \nabla f(W^t) \right \rangle.
\end{align*}
Since matrix $H^t_B$ is projection matrix, we have $(H^t_B)^\top H^t_B = (H^t_B)^2 = H^t_B$:
\begin{align*}
   f(W^{t+1})  \leq& f(W^t) -\gamma \left\langle \nabla f(W^t),  H^t_B \nabla f(W^t) \right \rangle + \frac{L}{2} \gamma^2 \left\langle  \nabla f(W^t),  H^t_B \nabla f(W^t) \right \rangle.
\end{align*}
Using the fact that $\gamma\leq \frac{1}{L}$ we have 
\begin{align*}
   f(W^{t+1})  \leq& f(W^t) -\frac{\gamma}{2} \left\langle \nabla f(W^t),  H^t_B \nabla f(W^t) \right \rangle.
\end{align*}
Taking expectation we get
\begin{align*}
   \mathbb{E} \left[ f(W^{t+1}) \mid W^t \right]  \leq&    \mathbb{E} \left[f(W^t) -\frac{\gamma}{2} \left\langle \nabla f(W^t),  H^t_B \nabla f(W^t) \right \rangle  \mid W^t \right]\\
   \leq & f(W^t) -\frac{\gamma}{2} \left\langle \nabla f(W^t),  \mathbb{E} \left[ H^t_B \right] \nabla f(W^t) \right \rangle  
   \end{align*}
   Using an Assumption \ref{asm:lambda} we have 
   \begin{align*}
   \mathbb{E} \left[ f(W^{t+1}) \mid W^t \right]  \leq&    \mathbb{E} \left[f(W^t) -\frac{\gamma}{2} \left\langle \nabla f(W^t),  H^t_B \nabla f(W^t) \right \rangle  \mid W^t \right]\\
   \leq & f(W^t) -\frac{\gamma}{2}\lambda_{\min}^{H_B}\left\| \nabla f(W^t)\right\|^2.
   \end{align*}
    Subtracting $f^\star$ from both sides we get
       \begin{align}
       \label{eq:GD for PL}
   \mathbb{E} \left[ f(W^{t+1}) \mid W^t \right]  - f^\star
   \leq & f(W^t) - f^\star -\frac{\gamma}{2}\lambda^{H_B}_{\min}\left\| \nabla f(W^t)\right\|^2.
   \end{align}
  Now we can rewrite as 
   \begin{align*}
       \frac{\gamma}{2}\lambda^{H_B}_{\min}\left\| \nabla f(W^t)\right\|^2 \leq    \left(  f(W^t) - f^\star  \right) - \left( \mathbb{E} \left[ f(W^{t+1}) \mid W^t \right]  - f^\star\right)
   \end{align*}
   Taking expectation and using tower property we obtain
      \begin{align*}
       \frac{\gamma}{2}\lambda^{H_B}_{\min}\mathbb{E}\left[\left\| \nabla f(W^t)\right\|^2 \right] \leq    e^t - e^{t+1},
   \end{align*}
   where $e^t =\mathbb{E} \left[ f(W^{t}) \right]  - f^\star $. Now we can sum these inequalities together and get
         \begin{align*}
     \sum^{T-1}_{t=0}  \frac{\gamma}{2}\lambda^{H_B}_{\min}\mathbb{E}\left[\left\| \nabla f(W^t)\right\|^2 \right] \leq  \sum^{T-1}_{t=0}  \left(  e^t - e^{t+1} \right),
   \end{align*}
   Using telescoping property of $e^t - e^{t+1}$ we get 
         \begin{align*}
     \sum^{T-1}_{t=0}  \frac{\gamma}{2}\lambda^{H_B}_{\min}\mathbb{E}\left[\left\| \nabla f(W^t)\right\|^2 \right] \leq e^0 - e^T.
   \end{align*}
   Once we divide by $T$ we obtain
            \begin{align*}
   \frac{1}{T}  \sum^{T-1}_{t=0}  \frac{\gamma}{2}\lambda^{H_B}_{\min}\mathbb{E}\left[\left\| \nabla f(W^t)\right\|^2 \right] &\leq \frac{e^0 - e^T}{T} \leq \frac{e^0 }{T}.
   \end{align*}
   Finally, we get
               \begin{align*}
   \frac{1}{T}  \sum^{T-1}_{t=0}  \mathbb{E}\left[\left\| \nabla f(W^t)\right\|^2 \right] \leq \frac{2 (f(W^0) - f^\star)}{\lambda^{H_B}_{\min} \gamma T}.
   \end{align*}
   Applying argument from \citet{danilova2022recent} we obtain the result for uniformly chosen point.

\end{proof}

\subsection{Proof of Theorem \ref{thm:PL-GD}}
\label{sec:GD-PL}


The proof is provided for Left Sketch (Definition \ref{def:left}). The result for Right Sketch (Definition \ref{def:right}) can be derived by following the same steps.


\begin{proof}

We start from the inequality \ref{eq:GD for PL}:
       \begin{align*}
   \mathbb{E} \left[ f(W^{t+1}) \mid W^t \right]  - f^\star
   \leq & f(W^t) - f^\star -\frac{\gamma}{2}\lambda^{H_B}_{\min}\left\| \nabla f(W^t)\right\|^2.
   \end{align*}
       Using PL condition $\left\| \nabla f(W^t) \right\|^2 \geq 2\mu\left( f(W^t) - f^\star \right)$ we have
     \begin{align*}
            \mathbb{E} \left[ f(W^{t+1}) \mid W^t \right]  - f^\star
   \leq & f(W^t) - f^\star -\gamma \mu\lambda^{H_B}_{\min}\left( f(W^t) - f^\star \right)\\
    \leq & \left(1 - \gamma \mu \lambda^{H_B}_{\min}\right)\left( f(W^t) - f^\star \right).
     \end{align*}
     Once we unroll the recursion we get
     \begin{align*}
         \mathbb{E}\left[ f(W^{T}) \right] -f^\star  \leq  \left(1 - \gamma \mu \lambda^{H_B}_{\min}\right)^T\left( f(W^0) - f^\star \right).
     \end{align*}
     In order to obtain $\varepsilon$ solution we need to take
     \begin{align*}
    T \geq \mathcal{O} \left( \frac{L}{\mu} \frac{1}{\lambda^{H_B}_{\min}} \log \frac{1}{\varepsilon}\right).
\end{align*}

\end{proof}

\clearpage
\section{Analysis of RAC-LoRA with Random Reshuffling} \label{sec:RR}

The previous results were obtained using full gradients. However, this approach is impractical in deep learning settings, where calculating full gradients is often infeasible. To analyze stochastic methods, we consider problem (\ref{eq:main}), where $f$ has the following special structure: 
\begin{align}
\label{eq:finite}
f(W^0 + \Delta W) := \frac{1}{N}\sum_{i=1}^{N} f_i(W^0+\Delta W) , 
\end{align}
where each function \(f_i\) represents the individual loss function for one sample and $N$ is total number of datapoints. Next, we analyze a practical variant of stochastic gradient descent (\algname{SGD}) known as Random Reshuffling (\algname{RR}), which involves sampling without replacement. In this method, the dataset is shuffled according to a permutation, ensuring that each training sample is used exactly once during each epoch. 

\algname{RR} is a variant of \algname{SGD} in which each data point is used exactly once per epoch, also known as \algname{SGD} with sampling without replacement. Many efforts have been made to explain why gradient methods with reshuffling perform so well in practice, across different types of problems. The convergence rates for incremental gradient methods with random reshuffling in convex optimization were first explored by \citet{nedic2001convergence} and later by \citet{bertsekas2011incremental}. In recent years, a lot of focus has shifted toward strongly convex problems, with studies showing that \algname{RR} can outperform \algname{SGD}. For example, \citet{recht2012toward} were among the first to analyze this for quadratic least squares problems. 

Researchers have also managed to improve results and remove some of the earlier assumptions, such as second-order smoothness, as seen in works by \citet{jain2019sgd}, \citet{safran2021random} and \citet{mishchenko2020random}. These studies introduced a new way to account for the random permutation's variance, making it easier to analyze both convex and strongly convex cases. There have even been extensions into non-convex settings, with results under the PL condition \citep{ahn2020sgd, nguyen2021unified} and general non-convex smooth cases \citep{lu2022general, mishchenko2020random, malinovsky2023random}.   More recently, tighter lower bounds for strongly convex and PL functions have been developed \citep{cha2023tighter}.

In recent years, there’s also been growing interest in applying these reshuffling techniques to distributed and federated learning, which is crucial for training large-scale, decentralized models \citep{yun2021minibatch, malinovsky2023server, sadiev2022federated,mishchenko2022proximal, cho2023convergence, malinovsky2022federated, malinovsky2023federated, horvath2022fedshuffle}


To analyze stochastic methods, we need to make assumptions about the variance. The standard assumption is that the variance is bounded:
\begin{assumption}
\label{asm:var}
There exist nonnegative constants $\sigma\geq 0$ such that for any $W^t \in \mathbb{R}^{m\times n}$ we have,
\begin{align*}
 \frac{1}{N} \sum_{i=1}^n\left\|\nabla f_i\left(W^t\right)-\nabla f\left(W^t\right)\right\|^2 \leq \sigma^2.   
\end{align*}

\end{assumption}


The proof is provided for Left Sketch (Definition \ref{def:left}). The result for Right Sketch (Definition \ref{def:right}) can be derived by following the same steps.

We consider a method belonging to the class of data permutation methods which is the \algname{RR} algorithm. In each epoch $t$ of \algname{RR}, we sample indices $\pi_0, \pi_1, \ldots, \pi_{N-1}$ without replacement from $\{1,2, \ldots, N\}$, i.e., $\left\{\pi_0, \pi_1, \ldots, \pi_{N-1}\right\}$ is a random permutation of the set $\{1,2 \ldots N\}$ and proceed with $N$ iterates of the form:
\begin{align*}
    W^t_{i+1} &= W^t_i - \gamma H^t_B \nabla f(W^t_i).
\end{align*}
We then set $W^{t+1} = W^t_N$ , and repeat the process for a total of $T$ \algname{LoRA} blocks. We can derive the effective step:
 \begin{align}   
 \label{eq:RR}
     W^{t+1} &= W^t - \gamma H^t_B\sum^{N-1}_{i=0} \nabla f(W^t_i)= W^t - \gamma H^t_B N \hat{g}^t,
\end{align}
where $\hat{g}^t = \frac{1}{N} \sum^{N-1}_{i=0} \nabla f(W^t_i). $ 

\subsection{Analysis of general non-convex setting}
\label{sec:RR-gen}
\begin{theorem}
\label{thm:RR}
Suppose that Assumption \ref{asm:L-smooth} and Assumption \ref{asm:lambda} hold. Suppose that a stepsize $\gamma > 0$ is chosen such that $\gamma\leq \frac{1}{2LN}$. We choose the output of the method $\widetilde{W}^T$ uniformly at random from $W^0, W^1,\ldots,W^{T-1}$ Then, the iterate $\widetilde{W}^T$ of \algname{RAC-LoRA} method (Algorithm \ref{alg:RAC-LoRA}) with \algname{RR} updates (Equation \ref{eq:RR}) satisfy
\begin{align*}
 \mathbb{E}\left[\left\| \nabla f(\widetilde{W}^T)\right\|^2 \right]\leq & \frac{2}{\gamma N T} \frac{f(W^0) - f^\star }{\left( 1- \lambda_{\max} \left[ \mathbb{E}\left[I-H^t\right] \right] - \frac{1 }{4}\lambda_{\max}^{H} \right)   } \\
 &+ \frac{L^2\gamma^2 \lambda_{\max}^{H} N \sigma^2}{\left( 1- \lambda_{\max} \left[ \mathbb{E}\left[I-H^t\right] \right] - \frac{1}{4}\lambda_{\max}^{H} \right)   }   .
\end{align*}
\end{theorem}

Remark: Notice that if we choose $\gamma = {\cal O}(1/T)$, the above result yields the rate ${\cal O}(1/T^2)$.

\begin{proof}
In this context, and in subsequent discussions, the notation \(\|\cdot\|\) refers to the Frobenius norm, while \(\langle \cdot \rangle\) denotes the inner product associated with the Frobenius norm.

Now we can apply the $L$-smoothness:
\begin{align*}
    f(W^{t+1}) &\leq f(W^t) + \left\langle \nabla f(W^t), W^{t+1} - W^t \right\rangle + \frac{L}{2}\left\| W^{t+1} - W^t \right\|^2\\
    & = f(W^t) + \left\langle \nabla f(W^t), -\gamma H^t_B N \hat{g}^t\right\rangle + \frac{L}{2}\left\| \gamma H^t_B N \hat{g}^t\right\|^2\\
    &=f(W^t)  -\gamma N\left\langle \nabla f(W^t), H^t_B  \hat{g}^t\right\rangle + \frac{L}{2} \gamma^2 N^2\left\| H^t_B \hat{g}^t\right\|^2\\
        &=f(W^t)  -\frac{\gamma N}{2}\left( \left\| \nabla f(W^t) \right\|^2 + \left\| H_B^t \hat{g}^t \right\|^2 - \left\| \nabla f(W^t) - H_B^t \hat{g}^t \right\|^2 \right) + \frac{L}{2} \gamma^2 N^2\left\| H_B^t \hat{g}^t\right\|^2\\
                &=f(W^t)  -\frac{\gamma N}{2}\left( \left\| \nabla f(W^t) \right\|^2 + \left\| H_B^t \hat{g}^t \right\|^2 - \left\| \nabla f(W^t) - H_B^t \hat{g}^t \right\|^2 \right) + \frac{L}{2} \gamma^2 N^2\left\| H_B^t \hat{g}^t\right\|^2\\
                & = f(W^t) - \frac{\gamma N}{2}\left\| \nabla f(W^t) \right\|^2 - \frac{\gamma N}{2}\left\| H_B^t \hat{g}^t \right\|^2\left( 1-\gamma L N \right) + \frac{\gamma N}{2} \left\| \nabla f(W^t) - H_B^t \hat{g}^t \right\|^2.
\end{align*}
Using $\gamma \leq \frac{1}{L N}$ we get 
\begin{align*}
    f(W^{t+1}) 
                & \leq f(W^t) - \frac{\gamma N}{2}\left\| \nabla f(W^t) \right\|^2 + \frac{\gamma N}{2} \left\| \nabla f(W^t) - H^t_B \hat{g}^t \right\|^2.
\end{align*}
Let us take expectation and subtract $f^\star$:
\begin{align*}
\mathbb{E}\left[    f(W^{t+1}) \mid W^t \right] - f^\star
                & \leq f(W^t) - f^\star - \frac{\gamma N}{2}\left\| \nabla f(W^t) \right\|^2 + \frac{\gamma N}{2} \mathbb{E} \left[ \left\| \nabla f(W^t) - H_B^t \hat{g}^t \right\|^2 \mid W^t \right].
\end{align*}
Let us consider the last term:
\begin{align*}
 \mathbb{E}&\left[    \left\| \nabla f(W^t) - H_B^t \hat{g}^t \right\|^2 \mid W^t \right]\\
 & =  \mathbb{E}\left[   \left\| \frac{1}{N} \sum^{N-1}_{i=0} \nabla f_{\pi_i}(W^t) - H^t_B\frac{1}{N}\sum^{N-1}_{i=0}\nabla f_{\pi_i}(W^t_i) \right\|^2 \mid W^t \right]\\
     & =  \mathbb{E}\left[\left\| \frac{1}{N} \sum^{N-1}_{i=0} \nabla f_{\pi_i}(W^t) + H^t_B\frac{1}{N} \sum^{N-1}_{i=0} \nabla f_{\pi_i}(W^t) - H^t_B\frac{1}{N} \sum^{N-1}_{i=0} \nabla f_{\pi_i}(W^t) - H_B^t \frac{1}{N}\sum^{N-1}_{i=0}\nabla f_{\pi_i}(W^t_i) \right\|^2 \mid W^t \right]
\end{align*}
Since $I-H^t_B$ and $H^t_B$ are projection matrices generating  perpendicular subspaces we have 
\begin{align*}
 \mathbb{E}&\left[    \left\| \nabla f(W^t) - H_B^t \hat{g}^t \right\|^2 \mid W^t \right]\\
     &=  \mathbb{E}\left[ \left\| \left(I - H_B^t\right) \nabla f(W^t) \right\|^2 + \left\| H_B^t \frac{1}{n} \sum^{ n-1}_{i=0}\left( \nabla f_{\pi_i}(W^t) - f_{\pi_i}(W^t_i) \right) \right\|^2 \mid W^t\right]\\
     & =  \mathbb{E}\left[ \left\langle \left(I - H_B^t\right) \nabla f(W^t) , \left(I - H_B^t\right) \nabla f(W^t) \right\rangle + \left\| H^t_B \frac{1}{N} \sum^{ N-1}_{i=0}\left( \nabla f_{\pi_i}(W^t) - f_{\pi_i}(W^t_i) \right) \right\|^2 \mid W^t\right].
\end{align*}
Using the property that $H^t_B$ and $I-H^t_B$ are projection matrices we obtain
\begin{align*}
 \mathbb{E}&\left[    \left\| \nabla f(W^t) - H^t \hat{g}^t \right\|^2 \mid W^t \right]\\
 & \leq \lambda_{\max} \left[ \mathbb{E}\left[I-H^t\right] \right] \left\| \nabla f(W^t) \right\|^2 + \mathbb{E}\left[ \lambda_{\max}[H^t] L^2 \frac{1}{N} \sum^{N-1}_{i=0} \left\| W^t - W^t_i  \right\|^2 \right].
 \end{align*}
 Since $\lambda_{\max}[H^t] = 1$ for projections matrix we get
 \begin{align*}
 \mathbb{E}&\left[    \left\| \nabla f(W^t) - H^t_B \hat{g}^t \right\|^2 \mid W^t \right] \leq \lambda_{\max} \left[ \mathbb{E}\left[I-H^t_B\right] \right] \left\| \nabla f(W^t) \right\|^2 +  L^2 \frac{1}{N} \sum^{N-1}_{i=0} \mathbb{E}\left[ \left\| W^t - W^t_i  \right\|^2 \mid W^t \right].
 \end{align*}
 Now let us consider the last term:
\begin{align*}
    \mathbb{E}\left[ \left\| W^t - W^t_k \right\|^2 \right] &= \gamma^2 \mathbb{E}\left[ \left\| \sum^{k-1}_{i=0} H^t_B \nabla f_{\pi_i}(W^t_i) \right\|^2 \mid W^t \right]\\
    & = \gamma^2  \mathbb{E}\left[ \left\| \sum^{k-1}_{i=0} H^t_B \left( \nabla f_{\pi_i}(W^t_i)  - \nabla f_{\pi_i}(W^t) \right) + \sum^{k-1}_{i=0} H^t_B \nabla f_{\pi_i}(W^t)  \right\|^2 \mid W^t\right]\\
    & \leq 2 \gamma^2 k \mathbb{E}\left[\sum^{k-1}_{i=0} \left(  \left\| H^t_B \left( \nabla f_{\pi_i}(W^t_i)  - \nabla f_{\pi_i}(W^t) \right)  \right\|^2  + 2\gamma^2 k^2 \left\| H^t_B \nabla f_{\pi_i}(W^t) \right\|^2\right)\mid W^t \right]\\
        & \leq 2 \gamma^2 k \mathbb{E}\left[ \sum^{k-1}_{i=0} \left( \lambda_{\max}\left[ H^t \right] \left\| W^t_i  - W^t   \right\|^2  + 2\gamma^2 k^2 \left\| H^t_B  \nabla f_{\pi_i}(W^t) \right\|^2\right) \mid W^t \right]\\
                & \leq 2 \gamma^2 k \mathbb{E}\left[ \sum^{k-1}_{i=0}  \left( \left\| W^t_i  - W^t   \right\|^2  + 2\gamma^2 k^2 \lambda_{\max}\left[ \mathbb{E}\left[H^t_B\right] \right]\left\|  \nabla f_{\pi_i}(W^t) \right\|^2\right)\mid W^t \right]. 
\end{align*}
Now, we are ready to sum the inequalities. By using \( \lambda_{\max}\left[ \mathbb{E}\left[H^t\right] \right] = \lambda_{\max}^{H_B} \) and applying Lemma 1 from \citet{mishchenko2020random} with Assumption \ref{asm:var}, we obtain:
 \begin{align*}
     \sum^{n-1}_{i=0}    \mathbb{E}\left[ \left\| W^t - W^t_i \right\|^2 \right]  \leq &    \mathbb{E}\left[  \sum^{N-1}_{i=0} \left( 2 \gamma^2 k \sum^{k-1}_{i=0}   \left\| W^t_i  - W^t   \right\|^2  + 2\gamma^2 k^2 \lambda_{\max}^{H_B}\left\|  \nabla f_{\pi_i}(W^t) \right\|^2 \right)\mid W^t \right]\\
      \leq & \gamma^2 L^2 N(N-1)      \sum^{N-1}_{i=0}    \mathbb{E}\left[ \left\| W^t - W^t_k \right\|^2 \right]\\
     &+ \frac{1}{3}\gamma^2(N-1)N(2N-1)\lambda_{\max}^{H_B}\left\|  \nabla f(W^t) \right\|^2 + \frac{1}{3} \lambda_{\max}^{H_B} \gamma^2 N(N+1) \sigma^2.
 \end{align*}
 Using $\gamma\leq \frac{1}{2LN}$ we get 
 \begin{align*}
          \sum^{n-1}_{i=0}    \mathbb{E}\left[ \left\| W^t - W^t_i \right\|^2 \right] &\leq \frac{4}{3} \left( 1-\gamma^2 L^2 N(N-1) \right)           \sum^{N-1}_{i=0}    \mathbb{E}\left[ \left\| W^t - W^t_i \right\|^2 \right] \\
          &\leq \frac{4}{3} \left( \frac{1}{3}\gamma^2(N-1)N(2N-1)\lambda_{\max}^{H_B}\left\|  \nabla f(W^t) \right\|^2  + \frac{1}{3} \lambda_{\max}^{H_B} \gamma^2 N(N+1) \sigma^2\right)\\
          &\leq \gamma^2 n^3 \lambda_{\max}^{H_B} \left\|  \nabla f(W^t) \right\|^2 +  \gamma^2 \lambda_{\max}^{H_B} N^2 \sigma^2
 \end{align*}
 Plugging to the previous bound we obtain:
 \begin{align*}
      \mathbb{E}\left[    \left\| \nabla f(W^t) - H^t \hat{g}^t \right\|^2 \mid W^t \right]&\leq \lambda_{\max} \left[ \mathbb{E}\left[I-H^t_B\right] \right] \left\| \nabla f(W^t) \right\|^2 +  L^2 \gamma^2 N^2 \lambda_{\max}^{H_B} \left\|  \nabla f(W^t) \right\|^2\\
      &+L^2\gamma^2 \lambda_{\max}^{H_B} N \sigma^2.
\end{align*}
Now we have the following 
\begin{align*}
\mathbb{E}\left[    f(W^{t+1}) \mid W^t \right] - f^\star
                \leq & f(W^t) - f^\star - \frac{\gamma N}{2}\left\| \nabla f(W^t) \right\|^2\\
                &+ \frac{\gamma N}{2} \left( \lambda_{\max} \left[ \mathbb{E}\left[I-H^t_B\right] \right] \left\| \nabla f(W^t) \right\|^2 +  L^2 \gamma^2 N^2 \lambda_{\max}^{H_B} \left\|  \nabla f(W^t) \right\|^2\right)\\
                &+ \frac{\gamma N}{2} L^2\gamma^2 \lambda_{\max}^{H_B} N \sigma^2.
\end{align*}
Using $\gamma\leq \frac{1}{2LN}$ we get 
\begin{align}
\label{eq:rr-pl-1}
\mathbb{E}\left[    f(W^{t+1}) \mid W^t \right] - f^\star
                \leq & f(W^t) - f^\star - \frac{\gamma N}{2}\left\| \nabla f(W^t) \right\|^2 \left( 1- \lambda_{\max} \left[ \mathbb{E}\left[I-H^t_B\right] \right] - \frac{1}{4}\lambda_{\max}^{H_B} \right)\\
                                &+ \frac{\gamma N}{2} L^2\gamma^2 \lambda_{\max}^{H_B} N \sigma^2.
\end{align}
After rearranging the terms, we have 
\begin{align*}
 \frac{\gamma N}{2}\left\| \nabla f(W^t) \right\|^2 \left( 1- \lambda_{\max} \left[ \mathbb{E}\left[I-H^t_B\right] \right] - \frac{1}{4}\lambda_{\max}^{H_B} \right)   
                \leq & \left(f(W^t) - f^\star \right) - \left(\mathbb{E}\left[    f(W^{t+1}) \mid W^t \right] - f^\star \right) \\
                                &+ \frac{\gamma N}{2} L^2\gamma^2 \lambda_{\max}^{H_B} N \sigma^2.
\end{align*}
Next, we have 
\begin{align*}
\left\| \nabla f(W^t) \right\|^2 
                \leq & \frac{2}{\gamma N} \frac{1}{\left( 1- \lambda_{\max} \left[ \mathbb{E}\left[I-H^t_B\right] \right] - \frac{1}{4}\lambda_{\max}^{H_B} \right)   } \left( \left(f(W^t) - f^\star \right) - \left(\mathbb{E}\left[    f(W^{t+1}) \mid W^t \right] - f^\star \right) \right) \\
                                &+ \frac{2}{\gamma N} \frac{1}{\left( 1- \lambda_{\max} \left[ \mathbb{E}\left[I-H^t_B\right] \right] - \frac{1}{4}\lambda_{\max}^{H_B} \right)   }   \frac{\gamma N}{2} L^2\gamma^2 \lambda_{\max}^{H_B} N \sigma^2.
\end{align*}
Using telescoping property and taking expectation we get
\begin{align*}
    \frac{1}{T} \sum^{T-1}_{t=0} \left\| \nabla f(W^t) \right\|^2 &\leq \frac{2}{\gamma N T} \frac{f(W^0) - f^\star }{\left( 1- \lambda_{\max} \left[ \mathbb{E}\left[I-H^t_B\right] \right] - \frac{1 }{4}\lambda_{\max}^{H_B} \right)   } \\
    & + \frac{L^2\gamma^2 \lambda_{\max}^{H_B} N \sigma^2}{\left( 1- \lambda_{\max} \left[ \mathbb{E}\left[I-H^t_B\right] \right] - \frac{1}{4}\lambda_{\max}^{H_B} \right)   }   .
\end{align*}
   Applying argument from \citet{danilova2022recent} we obtain the result for uniformly chosen point.
   \end{proof}

   \begin{corollary}
       \label{cor-RR-gen}
       Suppose that Assumption \ref{asm:L-smooth} and Assumption \ref{asm:lambda} hold. Suppose that a stepsize $\gamma > 0$ is chosen such that $\gamma\leq \frac{1}{2LN}$. Let the updates have a form of several gradient steps (variance $\sigma^2 = 0$)  We choose the output of the method $\widetilde{W}^T$ uniformly at random from $W^0, W^1,\ldots,W^{T-1}$ Then, the iterate $\widetilde{W}^T$ of \algname{RAC-LoRA} method (Algorithm \ref{alg:RAC-LoRA}) with several \algname{GD} updates (Equation \ref{eq:left_GD}) satisfy
\begin{align*}
 \mathbb{E}\left[\left\| \nabla f(\widetilde{W}^T)\right\|^2 \right]\leq & \frac{2}{\gamma N T} \frac{f(W^0) - f^\star }{\left( 1- \lambda_{\max} \left[ \mathbb{E}\left[I-H^t\right] \right] - \frac{1 }{4}\lambda_{\max}^{H} \right)   }.
\end{align*}
   \end{corollary}
   
Given that the step size is divided by the number of gradient steps allocated for each \algname{LoRA} block, employing multiple gradient steps for a single \algname{LoRA} block does not provide any significant benefits. This observation suggests that a single gradient step is adequate for each \algname{LoRA} block. Therefore, in practical applications, it is more advantageous to utilize only one epoch per \algname{LoRA} block within the training chain. This approach not only streamlines the training process but also optimizes computational efficiency, allowing for more effective resource allocation without compromising the performance of the model.

\subsection{Analysis of Polyak-Łojasiewicz setting}
\label{sec:RR-PL}
Next, we establish the convergence rate for the Polyak-Łojasiewicz setting (Assumption \ref{asm:PL}).

\begin{theorem}
\label{thm:RR-PL}
Suppose that Assumption \ref{asm:L-smooth}, Assumption \ref{asm:PL} and Assumption \ref{asm:lambda} hold. Suppose that a stepsize $\gamma \geq 0$ is chosen such that $\gamma\leq \frac{1}{2NL}$. Then, the iterates of \algname{RAC-LoRA} method (Algorithm~\ref{alg:RAC-LoRA}) with \algname{RR} updates (Equation \ref{eq:RR})  satisfy
\begin{align*}
    \mathbb{E}\left[    f(W^{T})  - f^\star\right]
                \leq & \left(1 - \gamma N \mu\left( 1- \lambda_{\max} \left[ \mathbb{E}\left[I-H^t_B\right] \right] - \frac{1}{4}\lambda_{\max}^{H_B} \right) \right)^T\mathbb{E}\left[f(W^0) - f^\star \right]\\
                                &+ \frac{L^2\gamma^2 \lambda_{\max}^{H_B} N \sigma^2}{2\left( 1- \lambda_{\max} \left[ \mathbb{E}\left[I-H^t_B\right] \right] - \frac{1}{4}\lambda_{\max}^{H_B} \right)} .
\end{align*}
\end{theorem}
\begin{proof}

We start from Equation \ref{eq:rr-pl-1}:
\begin{align*}
\mathbb{E}\left[    f(W^{t+1}) \mid W^t \right] - f^\star
                \leq & f(W^t) - f^\star\\
                &- \frac{\gamma N}{2}\left\| \nabla f(W^t) \right\|^2 \left( 1- \lambda_{\max} \left[ \mathbb{E}\left[I-H^t_B\right] \right] - \frac{1}{4}\lambda_{\max}^{H_B} \right)\\
                                &+ \frac{\gamma N}{2} L^2\gamma^2 \lambda_{\max}^{H_B} N \sigma^2.
\end{align*}
       Using PL condition $\left\| \nabla f(W^t) \right\|^2 \geq 2\mu\left( f(W^t) - f^\star \right)$ we have
\begin{align*}
\mathbb{E}\left[    f(W^{t+1}) \mid W^t \right] - f^\star
                \leq & f(W^t) - f^\star\\
                &- \gamma N\mu\left( f(W^t) - f^\star \right) \left( 1- \lambda_{\max} \left[ \mathbb{E}\left[I-H^t_B\right] \right] - \frac{1}{4}\lambda_{\max}^{H_B} \right)\\
                                &+ \frac{\gamma N}{2} L^2\gamma^2 \lambda_{\max}^{H_B} N \sigma^2.
\end{align*}
Taking full expectation we obtain:
\begin{align*}
    \mathbb{E}\left[    f(W^{t+1})  - f^\star\right]
                \leq & \left(1 - \gamma N \mu\left( 1- \lambda_{\max} \left[ \mathbb{E}\left[I-H^t_B\right] \right] - \frac{1}{4}\lambda_{\max}^{H_B} \right) \right)\mathbb{E}\left[f(W^t) - f^\star \right]\\
                                &+ \frac{\gamma N}{2} L^2\gamma^2 \lambda_{\max}^{H_B} N \sigma^2.
\end{align*}
After unrolling the recursion we obtain
\begin{align*}
    \mathbb{E}\left[    f(W^{T})  - f^\star\right]
                \leq & \left(1 - \gamma N \mu\left( 1- \lambda_{\max} \left[ \mathbb{E}\left[I-H^t_B\right] \right] - \frac{1}{4}\lambda_{\max}^{H_B} \right) \right)^T\mathbb{E}\left[f(W^0) - f^\star \right]\\
                                &+ \frac{L^2\gamma^2 \lambda_{\max}^{H_B} N \sigma^2}{2\left( 1- \lambda_{\max} \left[ \mathbb{E}\left[I-H^t_B\right] \right] - \frac{1}{4}\lambda_{\max}^{H_B} \right)} .
\end{align*}
This finishes the proof.
\end{proof}

\begin{corollary}
\label{cor-RR-PL}
    Suppose that Assumption \ref{asm:L-smooth}, Assumption \ref{asm:PL} and Assumption \ref{asm:lambda} hold. Let the updates have a form of several gradient steps (variance $\sigma^2 = 0$) Suppose that a stepsize $\gamma \geq 0$ is chosen such that $\gamma\leq \frac{1}{2NL}$. Then, the iterates of \algname{RAC-LoRA} method (Algorithm \ref{alg:RAC-LoRA}) with several \algname{GD} updates (Equation \ref{eq:left_GD})  satisfy

    \begin{align*}
    \mathbb{E}\left[    f(W^{T})  - f^\star\right]
                \leq & \left(1 - \gamma N \mu\left( 1- \lambda_{\max} \left[ \mathbb{E}\left[I-H^t_B\right] \right] - \frac{1}{4}\lambda_{\max}^{H_B} \right) \right)^T\mathbb{E}\left[f(W^0) - f^\star \right].
\end{align*}
\end{corollary}
Since the step size is divided by the number of gradient steps for each \algname{LoRA} block, using multiple gradient steps does not offer significant advantages. Thus, a single gradient step per \algname{LoRA} block is sufficient. Practically, it is more efficient to use only one epoch per \algname{LoRA} block in the training chain.

\clearpage
\section{Analysis of RAC-LoRA with SGD under the arbitrary data sampling paradigm}\label{sec:SGD}

In the previous section, we introduced the Random Reshuffling (\algname{RR}) method, where each data point is used exactly once during each epoch, also known as sampling without replacement. This method has demonstrated strong empirical performance across various optimization tasks. However, in this section, we shift our focus to the \algname{RAC-LoRA} framework, where Stochastic Gradient Descent (\algname{SGD}) is applied with a more general, arbitrary data sampling procedure, allowing for broader flexibility in how data is selected and used during training.

The analysis of general sampling schemes in \algname{SGD} has garnered significant attention in the literature, particularly in understanding its impact on convergence rates and optimization performance across different problem classes. For strongly convex functions, general sampling methods have been rigorously studied in works such as \citet{gower2019sgd}, which provide detailed convergence guarantees and bounds. In the case of general convex optimization problems, \citet{khaled2023unified} offer a thorough analysis of the performance of \algname{SGD} under various sampling strategies. Furthermore, for non-convex settings, both \citet{khaled2020better} and \citet{demidovich2023guide} have explored how general sampling procedures influence the convergence behavior and optimization efficiency of \algname{SGD}, shedding light on its applicability to a wide range of machine learning tasks.

In the following sections, we build on these foundational studies to examine how the flexibility of general sampling in the \algname{RAC-LoRA} framework can lead to improved convergence in certain scenarios, while also maintaining robust performance across different convexity settings.

To conduct this analysis, we introduce a general assumption that extends the standard assumptions presented in \citet{khaled2020better}. 

The proof is provided for Right Sketch (Definition \ref{def:right}). The result for Left Sketch (Definition \ref{def:left}) can be derived by following the same steps.

\begin{assumption}[Expected smoothness] The second moment of the stochastic gradient satisfies
\label{asm:ABC}
$$
\mathbb{E}\left[\|g(W)\|^2\right] \leq 2 A_1\left(f(W)-f^{\mathrm{inf}}\right)+B_1 \cdot\|\nabla f(W)\|^2+C_1
$$

for some $A, B, C \geq 0$ and all $W \in \mathbb{R}^{m\times n}$.
\end{assumption}

Now we can also do stochastic analysis. Let us consider the \algname{SGD} update for \algname{LoRA} method:
$$ \Delta W = \frac{\alpha}{r} \hat{B} A_{S}, $$
\begin{align*}
W^{t+1} = W^t + \frac{\alpha}{r} \hat{B}^t A_{S}^t, \qquad \hat{B}^t = -\gamma g(W^t)(A^t_{S})^\top \left(A^t_{S} (A^t_{S})^\top\right)^\dagger
\end{align*}
Now we have 
\begin{align}
\label{eq:SGD}
    W^{t+1} &= W^t  -\gamma g(W^t)(A^t_{S})^\top \left(A^t_{S}(A^t_{S})^\top\right)^\dagger A^t_{S}
    = W^t  -\gamma  g(W^t){H}^t_A.
\end{align}
\subsection{Analysis of general non-convex setting}
\label{sec:SGD-gen}
\begin{theorem}
\label{thm:SGD-gen}
Suppose that Assumption \ref{asm:L-smooth} and Assumption \ref{asm:lambda} hold. Suppose that a stepsize $\gamma > 0$ is chosen such that $\gamma\leq \min \left[ 1 / \sqrt{L A_1 \lambda_{\max}^{H} T}, 1/\left(LB_1 \frac{\lambda_{\max}^{H_A} }{\lambda_{\min}^{H_A} }\right) \right]$. Then, the iterate ${W}^T$ of \algname{RAC-LoRA} method (Algorithm \ref{alg:RAC-LoRA}) with \algname{SGD} updates (Equation \ref{eq:SGD}) satisfy
    \begin{align*}
 \min_{0\leq t \leq T-1} \mathbb{E} \left[ \left\| \nabla f(W^T) \right\|^2 \right] \leq  \frac{6} { \lambda_{\min}^{H_A} \gamma T} \left( f(W^0) - f^\star \right)+ LC_1 \gamma\frac{\lambda_{\max}^{H_A} }{\lambda_{\min}^{H_A} }.
\end{align*}
\end{theorem}

Remark: Notice that if we choose $\gamma = {\cal O}(1/\sqrt{T})$, the above result yields the rate ${\cal O}(1/\sqrt{T})$.

\begin{proof}
We start from $L$-smoothness:
\begin{align*}
    f(W^{t+1}) \leq & f(W^t) + \left\langle \nabla f(W^t), W^{t+1} - W^t \right\rangle + \frac{L}{2}\left\| W^{t+1} - W^t \right\|^2\\
    =& f(W^t) + \left\langle \nabla f(W^t), -\gamma g(W^t) {H}^t_A \right\rangle + \frac{L}{2}\left\| -\gamma g(W^t) {H}^t_A \right\|^2\\
    =& f(W^t) -\gamma  \left\langle \nabla f(W^t), g(W^t) {H}^t_A\right\rangle + \frac{L}{2}\left\| -\gamma g(W^t) {H}^t_A \right\|^2.
\end{align*}
Let us take conditional expectation:
\begin{align*}
    \mathbb{E}\left[f(W^{t+1}) \mid W^t \right] \leq & f(W^t) -\gamma  \mathbb{E}\left[\left\langle \nabla f(W^t), g(W^t) {H}^t_A\right\rangle \mid W^t\right]\\
    &+ \frac{L}{2}\mathbb{E}\left[\left\| -\gamma g(W^t) {H}^t_A \right\|^2 \mid W^t \right].
\end{align*}
Using that $g(W^t)$ and ${H}^t_A$ are independent, so we have 
\begin{align*}
    \mathbb{E}\left[f(W^{t+1}) \mid W^t \right] \leq & f(W^t) -\gamma  \left\langle \nabla f(W^t), \mathbb{E}\left[g(W^t) \right] \mathbb{E}\left[{H}^t_A \right]\right\rangle\\
    &+ \frac{L}{2}\mathbb{E}\left[\left\| -\gamma g(W^t) {H}^t_A \right\|^2 \mid W^t  \right]\\
     \leq & f(W^t) -\gamma  \left\langle \nabla f(W^t), \mathbb{E}\left[g(W^t) \right] \mathbb{E}\left[{H}^t_A\right]\right\rangle\\
     &+ \gamma^2\frac{L}{2}\mathbb{E}\left[\left\langle  g(W^t) {H}^t_A ,  g(W^t) {H}^t_A \right\rangle\mid W^t  \right]\\
     \leq & f(W^t) -\gamma \lambda_{\min}\left[\mathbb{E}\left[{H}^t_A \right]\right] \left\| \nabla f(W^t) \right\|^2\\
     &+ \gamma^2\frac{L}{2}\mathbb{E}\left[\left\langle  g(W^t) {H}^t_A ,  g(W^t) {H}^t_A \right\rangle\mid W^t  \right].
\end{align*}
Using the property of projection matrix ${H}^t_A$, we have 
\begin{align*}
    \mathbb{E}\left[f(W^{t+1}) \mid W^t \right]
     \leq & f(W^t) -\gamma \lambda_{\min}\left[\mathbb{E}\left[{H}^t_A \right]\right] \left\| \nabla f(W^t) \right\|^2\\
     &+ \gamma^2\frac{L}{2} \lambda_{\max}\left[ \mathbb{E}\left[ {H}^t_A  \right] \right]\mathbb{E}\left[ \left\|  g(W^t)  \right\|^2 \right] .
\end{align*}
Now we need to use assumption on stochastic gradients. We will use the most general assumption: ABC -- assumption:
\begin{align*}
    \mathbb{E}\left[ \|g(W^t)\|^2 \right] \leq 2A_1(f(W^t) - f^\star) + B_1\left\|\nabla f(W^t) \right\|^2 +C_1.
\end{align*}
Now we have 
\begin{align*}
        \mathbb{E}\left[f(W^{t+1}) \mid W^t \right] - f^\star
     \leq & f(W^t) - f^\star -\gamma \lambda_{\min}\left[\mathbb{E}\left[{H}^t_A \right]\right] \left\| \nabla f(W^t) \right\|^2\\
     +& \gamma^2\frac{L}{2} \lambda_{\max}\left[ \mathbb{E}\left[ {H}^t_A  \right] \right]\left(2A_1(f(W^t) - f^\star) + B_1\left\|\nabla f(W^t) \right\|^2 +C_1.\right) .
\end{align*}
Combining these terms together we get 
\begin{align}
\label{eq:sgd-ref}
\notag        \mathbb{E}\left[f(W^{t+1}) \mid W^t \right] - f^\star
     \leq & \left(f(W^t) - f^\star\right) \left( 1+\gamma^2 A_1 L \lambda_{\max}\left[\mathbb{E}\left[{H}^t_A \right]\right]  \right)\\
  \notag   &-\gamma \lambda_{\min}\left[\mathbb{E}\left[{H}^t_A \right]\right] \left\| \nabla f(W^t) \right\|^2\left(1 - \gamma \frac{L}{2} \frac{\lambda_{\max}\left[\mathbb{E}\left[{H}^t_A \right]\right] }{\lambda_{\min}\left[\mathbb{E}\left[{H}^t_A \right]\right] } B_1 \right)\\
     &+ \gamma^2\frac{L}{2} \lambda_{\max}\left[ \mathbb{E}\left[ {H}^t_A  \right] \right]C_1 .
\end{align}
Using condition on stepsize: $1 - \gamma \frac{LB_1}{2} \frac{\lambda_{\max}\left[\mathbb{E}\left[{H}^t_A \right]\right] }{\lambda_{\min}\left[\mathbb{E}\left[{H}^t_A \right]\right] }  \geq \frac{1}{2}$ we get 
\begin{align*}
        \mathbb{E}\left[f(W^{t+1}) \mid W^t \right] - f^\star
     \leq & \left(f(W^t) - f^\star\right) \left( 1+\gamma^2 A_1 L \lambda_{\max}\left[\mathbb{E}\left[{H}^t_A \right]\right]  \right)\\
     &-\frac{1}{2}\gamma \lambda_{\min}\left[\mathbb{E}\left[{H}^t_A \right]\right] \left\| \nabla f(W^t) \right\|^2\\
     &+ \gamma^2\frac{L}{2} \lambda_{\max}\left[ \mathbb{E}\left[ {H}^t_A \right] \right]C_1 .
\end{align*}
Using tower property of expectation we obtain 
\begin{align*}
        \mathbb{E}\left[f(W^{t+1})  - f^\star \right]
     \leq & \mathbb{E}\left[f(W^t) - f^\star \right] \left( 1+\gamma^2 A_1 L \lambda_{\max}\left[\mathbb{E}\left[{H}^t_A\right]\right]  \right)\\
     &-\frac{1}{2}\gamma \lambda_{\min}\left[\mathbb{E}\left[{H}^t_A \right]\right] \mathbb{E} \left[ \left\| \nabla f(W^t) \right\|^2 \right]\\
     &+ \gamma^2\frac{L}{2} \lambda_{\max}\left[ \mathbb{E}\left[ {H}^t_A  \right] \right]C_1 .
\end{align*}
Let us define $\delta^t = \mathbb{E}\left[ f(W^t) - f^\star \right]$ and $r^t = \mathbb{E} \left[ \left\| \nabla f(W^t) \right\|^2 \right]$, after reshuffling of terms we obtain
\begin{align*}
         \frac{1}{2}\gamma \lambda_{\min}\left[\mathbb{E}\left[{H}^t_A\right]\right] \mathbb{E} \left[ \left\| \nabla f(W^t) \right\|^2 \right] \leq & \left( 1+\gamma^2 A_1 L \lambda_{\max}\left[\mathbb{E}\left[{H}^t_A\right]\right]  \right) \delta^t - \delta^{t+1}\\
         &+ \gamma^2\frac{LC_1}{2} \lambda_{\max}\left[\mathbb{E}\left[{H}^t_A \right]\right].
\end{align*}
Let use fix $w^{-1}>0$ and define $w^t = \frac{w^{t-1}}{1+L\gamma^2 A \lambda_{\max}\left[\mathbb{E}\left[{H}^t_A \right]\right] }$ for all $t\geq 0$. Multiplying by $\frac{w^t}{\gamma }$, 
\begin{align*}
         \frac{1}{2} w^t r^t \lambda_{\min}\left[\mathbb{E}\left[{H}^t_A\right]\right]  \leq & \frac{w^t}{\gamma}\left( 1+\gamma^2 A_1 L \lambda_{\max}\left[\mathbb{E}\left[{H}^t_A\right]\right]  \right) \delta^t - \frac{w^t}{\gamma}\delta^{t+1} + \gamma\frac{LC_1}{2} \lambda_{\max}\left[\mathbb{E}\left[{H}^t_A \right]\right].
\end{align*}
Now we obtain 
\begin{align*}
         \frac{1}{2} w^t r^t \lambda_{\min}\left[\mathbb{E}\left[{H}^t_A\right]\right]  \leq & \frac{w^{t-1}}{\gamma}\delta^t - \frac{w^t}{\gamma}\delta^{t+1} + \gamma\frac{LC_1}{2} \lambda_{\max}\left[\mathbb{E}\left[{H}^t_A\right]\right] w^t.
\end{align*}
Summing up both sides as $t = 0,1,\ldots, T-1$ we have, 
\begin{align*}
         \frac{1}{2} \sum^{T-1}_{t=0} w^t r^t \lambda_{\min}\left[\mathbb{E}\left[{H}^t_A\right]\right]  \leq & \frac{w^{-1}}{\gamma}\delta^0 - \frac{w^{T-1}}{\gamma}\delta^{T} + \gamma\frac{LC_1}{2} \lambda_{\max}\left[\mathbb{E}\left[{H}^t_A\right]\right] \sum^{T-1}_{t=0} w^t\\
         \leq & \frac{w^{-1}}{\gamma}\delta^0 + \gamma\frac{LC_1}{2} \lambda_{\max}\left[\mathbb{E}\left[{H}^t_A\right]\right] \sum^{T-1}_{t=0} w^t.
\end{align*}
Let us define $W^T = \sum^{T-1}_{t=0} w^t$. Dividing both sides by $W^T$ we have, 
\begin{align*}
    \frac{1}{2} \min_{0\leq t \leq T-1} r^t\leq \frac{1}{W^T} \sum^{T-1}_{t=0} w^t r^t \leq \frac{w^{-1}}{W^T} \frac{\delta^0}{\gamma} \frac{1}{\lambda_{\min}\left[\mathbb{E}\left[{H}^t_A\right]\right] } + \frac{LC_1 \gamma}{2} \frac{\lambda_{\max}\left[\mathbb{E}\left[{H}^t_A\right]\right] }{\lambda_{\min}\left[\mathbb{E}\left[{H}^t_A\right]\right] }.
\end{align*}
Note that, 
\begin{align*}
    W^T = \sum^{T-1}_{t=0} w^t \geq \sum^{T-1}_{t=0}  \min_{0\leq i \leq T-1} w^i = T w^{T-1} = \frac{T w^{-1}}{\left(1+L\gamma^2 A \lambda_{\max}\left[\mathbb{E}\left[{H}^t_A \right]\right] \right)^T}.
\end{align*}
Using this we get
\begin{align*}
        \frac{1}{2} \min_{0\leq t \leq T-1} r^t \leq  \frac{\left(1+L\gamma^2 A_1 \lambda_{\max}\left[\mathbb{E}\left[{H}^t_A\right]\right] \right)^T} { \lambda_{\min}\left[\mathbb{E}\left[{H}^t_A \right]\right] \gamma T} \delta^0+ \frac{LC_1 \gamma}{2} \frac{\lambda_{\max}\left[\mathbb{E}\left[{H}^t_A\right]\right] }{\lambda_{\min}\left[\mathbb{E}\left[{H}^t_A\right]\right] }.
\end{align*}

Using the fact that $1+x \leq \exp (x)$, we have that

$$
\left(1+L \gamma^2 A_1 \lambda_{\max}\left[\mathbb{E}\left[{H}^t_A\right]\right]  \right)^T\leq \exp \left(L \gamma^2 A_1 \lambda_{\max}\left[\mathbb{E}\left[{H}^t_A\right]\right] T\right) \leq \exp (1) \leq 3
$$

where the second inequality holds because $\gamma \leq 1 / \sqrt{L A_1 \lambda_{\max}\left[\mathbb{E}\left[{H}^t_A\right]\right] T}$ by assumption. Substituting we get,

\begin{align*}
 \min_{0\leq t \leq T-1} r^t \leq  \frac{6} { \lambda_{\min}\left[\mathbb{E}\left[{H}^t_A \right]\right] \gamma T} \left( f(W^0) - f^\star \right)+ LC_1 \gamma\frac{\lambda_{\max}^{H_A} }{\lambda_{\min}^{H_A} }.
\end{align*}

\end{proof}

\subsection{Analysis of Polyak-Łojasiewicz setting }
\label{sec:SGD-PL}
In this section we provide analysis of \algname{RAC-LoRA} method with general \algname{SGD} update under Polyak-Łojasiewicz condition (Assumption~\ref{asm:PL}).

\begin{theorem}
\label{thm:SGD-PL}
Suppose that Assumption \ref{asm:L-smooth}, Assumption \ref{asm:PL} and Assumption \ref{asm:lambda} hold. Suppose that a stepsize $\gamma \geq 0$ is chosen such that $\gamma \leq \min \left[\frac{\mu}{2 A_1 L \frac{\lambda_{\max}^{H_A} }{\lambda_{\min}^{H_A} }},  1/\left(LB_1 \frac{\lambda_{\max}^{H_A} }{\lambda_{\min}^{H_A} }\right)\right]$. Then, the iterates of \algname{RAC-LoRA} method (Algorithm \ref{alg:RAC-LoRA}) with \algname{SGD} updates (Equation \ref{eq:SGD})  satisfy
     \begin{align*}
         \mathbb{E}\left[ f(W^{T}) \right] -f^\star  \leq  \left(1 - \gamma \mu \lambda^H_{\min}\right)^T\left( f(W^0) - f^\star \right).
     \end{align*}
\end{theorem}

\begin{proof}
    
We start from \ref{eq:sgd-ref}:
\begin{align*}
        \mathbb{E}\left[f(W^{t+1}) \mid W^t \right] - f^\star
     \leq & \left(f(W^t) - f^\star\right) \left( 1+\gamma^2 A_1 L \lambda_{\max}\left[\mathbb{E}\left[{H}^t_A \right]\right]  \right)\\
     &-\gamma \lambda_{\min}\left[\mathbb{E}\left[{H}^t_A \right]\right] \left\| \nabla f(W^t) \right\|^2\left(1 - \gamma \frac{L}{2} \frac{\lambda_{\max}\left[\mathbb{E}\left[{H}^t_A\right]\right] }{\lambda_{\min}\left[\mathbb{E}\left[{H}^t_A \right]\right] } B_1 \right)\\
     &+ \gamma^2\frac{L}{2} \lambda_{\max}\left[ \mathbb{E}\left[ {H}^t_A  \right] \right]C_1 .
\end{align*}
Using $\left(1 - \gamma \frac{L}{2} \frac{\lambda_{\max}\left[\mathbb{E}\left[{H}^t_A \right]\right] }{\lambda_{\min}\left[\mathbb{E}\left[{H}^t_A\right]\right] } B_1 \right) \geq \frac{3}{4}$ and PL condition we have 
\begin{align*}
        \mathbb{E}\left[f(W^{t+1}) \mid W^t \right] - f^\star
     \leq & \left(f(W^t) - f^\star\right) \left( 1 - \frac{3}{2} \gamma \mu \lambda_{\min}\left[\mathbb{E}\left[{H}^t_A\right]\right] +\gamma^2 A_1 L \lambda_{\max}\left[\mathbb{E}\left[{H}^t_A \right]\right]  \right)\\
     &+ \gamma^2\frac{L}{2} \lambda_{\max}\left[ \mathbb{E}\left[ {H}^t_A \right] \right]C_1.
\end{align*}
Using that $L A_1 \gamma \lambda_{\max}\left[\mathbb{E}\left[{H}^t_A \right]\right] \leq \frac{\mu}{2} \lambda_{\min}\left[\mathbb{E}\left[{H}^t_A \right]\right] $ we obtain
\begin{align*}
            \mathbb{E}\left[f(W^{t+1}) \mid W^t \right] - f^\star
     \leq & \left(f(W^t) - f^\star\right) \left( 1 -  \gamma \mu \lambda_{\min}\left[\mathbb{E}\left[{H}^t_A \right]\right]  \right)+ \gamma^2\frac{L}{2} \lambda_{\max}\left[ \mathbb{E}\left[ {H}^t_A \right] \right]C_1.
\end{align*}
Taking full expectation and using tower property we obtain:
\begin{align*}
            \mathbb{E}\left[f(W^{t+1})   - f^\star \right]
     \leq &  \mathbb{E}\left[  f(W^t) - f^\star\right] \left(1 -  \gamma \mu \lambda_{\min}\left[\mathbb{E}\left[{H}^t_A\right]\right]  \right)+ \gamma^2\frac{L}{2} \lambda_{\max}\left[ \mathbb{E}\left[ {H}^t_A  \right] \right]C_1.
\end{align*}
Once we unroll the recursion we obtain
\begin{align*}
            \mathbb{E}\left[f(W^{T})   - f^\star \right]
     \leq &  \mathbb{E}\left[  f(W^0) - f^\star\right] \left(1 -  \gamma \mu \lambda_{\min}^{H_A}  \right)^T+ \gamma\frac{L}{2\mu \lambda_{\min}^{H_A}} \lambda_{\max}^{H_A}C_1.
\end{align*}
\end{proof}

\clearpage
\section{Federated Learning setting} \label{sec:FL}

	\begin{algorithm}[t]
	\caption{Federated Randomized Asymmetric Chain of LoRA (\algname{Fed-RAC-LoRA})}\label{alg:Fed-RAC-LoRA}
	\begin{algorithmic}[1]
		\STATE	\textbf{Parameters:} initial pre-trained model $W^0$, rank $r$, learning rate $\gamma  > 0$, scaling factor $\alpha$, server stepsize $\beta >0$ number of modules in chain $T$,  sample distribution $\mathcal{D}^B_S$ or $\mathcal{D}^A_S$.

  \FOR{$t = 0, 1, \ldots , T-1$}\do\\ 
  \STATE Sample a subset (cohort) of clients $S^t$
     \STATE  \qquad (Option 1) \quad Sample a matrix $B^t_S$\qquad \qquad  (Option 2) \quad Sample a matrix $A^t_S$
  \STATE Send the model $W^t$ and fixed matrix (Option 1) $B^t_S$ or 
  (Option 2) $A^t_S$ to clients
  		\FOR{$m \in S^t$}\do\\      
		\STATE  Solve subproblem $$ \text{(Option 1) }  \hat{A}^t_m \approx \min \limits_{A} f_m(W^t+ \frac{\alpha}{r} B^t_S A)\qquad \text{(Option 2) }  \hat{B}^t_m \approx \min \limits_{B} f_m(W^t+ \frac{\alpha}{r} B A^t_S)$$
  \STATE Send the updates to server (Option 1) $\hat{A}^t_m $ or (Option 2) $\hat{B}^t_m$
  \ENDFOR
		\STATE  Merge the updates 
  \STATE $$ \text{(Option 1)} \quad W^{t+1} = W^t + \beta \frac{\alpha}{r} B^t_S \frac{1}{C}\sum_{m \in S^t}\hat{A}^t_m $$
    \STATE $$\text{(Option 2)} \quad W^{t+1} = W^t + \beta \frac{\alpha}{r} \frac{1}{C}\sum_{m \in S^t}\hat{B}^t_m A^t_S$$
		\ENDFOR
	\end{algorithmic}
\end{algorithm}

We consider the main optimization problem (\ref{eq:main}), with $f$ having the   double finite-sum structure
\begin{equation}\label{eq:fed}
       f(W^0 + \Delta W) := \frac{1}{M}\sum_{m=1}^M \underbrace{ \frac{1}{N} \sum_{i=1}^N f_{m,i}(W^0 + \Delta W)}_{f_m(W^0 + \Delta W)},
\end{equation}
where \(M\) is the total number of clients and \(N\) is the number of data points on each client. In the context of Federated Learning, each client maintains its own local loss function \(f_m\), which also follows a finite-sum structure, reflecting the client's local data. This formulation captures the decentralized nature of the learning process, where each client performs computations based on their local dataset.

Federated Learning (FL) \citep{ konecnyFL, kairouz2021advances}  is a distributed machine learning framework that enables multiple devices or clients to collaboratively train a shared model without sending their raw data to a central server. In contrast to traditional machine learning, where data is centralized for model training, Federated Learning allows each client to train a local model using its own data. The clients then share only the updated model parameters with a central server or aggregator. The server aggregates these updates to form a new global model, which is then redistributed to the clients for further iterations of the process \citep{konecnyFL}. Local Training (LT) is a key component of Federated Learning (FL), in which each participating client conducts several local optimization steps before synchronizing their model parameters with the central server.

The analysis of LT marked a significant advancement by eliminating the need for data homogeneity assumptions, as demonstrated by \citet{khaled2019first, khaled2020tighter}. However, later studies by \citet{woodworth2020minibatch} and \citet{glasgow2022sharp} revealed that \algname{LocalSGD} (also known as \algname{FedAvg}) has no communication complexity advantage over  \algname{MinibatchSGD} in heterogeneous data settings. Additionally, \citet{malinovskiy2020local} analyzed LT methods for general fixed-point problems, while \citet{koloskova2020unified} explored decentralized aspects of LT.

Although removing the data homogeneity requirement was a major breakthrough, the results were somewhat discouraging, as they indicated that LT-enhanced \algname{GD}, or \algname{LocalGD}, exhibits a sublinear convergence rate, which is worse than the linear convergence rate of vanilla \algname{GD} \citep{woodworth2020local}. The impact of server-side step sizes was further explored by \citet{malinovsky2023server} and \citet{charles2020outsized}.

Subsequent LT methods aimed to achieve linear convergence by addressing client drift, which had hindered earlier approaches. \algname{Scaffold}, introduced by \citet{karimireddy2020scaffold}, was the first to successfully mitigate client drift and achieve a linear convergence rate. Similar methods were later proposed by \citet{gorbunov2021local}. Although this was a significant breakthrough, these methods still have slightly higher or equal communication complexity compared to vanilla \algname{GD}.

\citet{mishchenko2022proxskip} recently introduced the \algname{ProxSkip} method, a simple yet effective approach to Local Training that achieves provable communication acceleration in the smooth strongly convex regime, even with heterogeneous data. In a follow-up article, \citet{malinovsky2022variance} expanded on \algname{ProxSkip}, presenting a broad variance reduction framework. \citet{condat2022randprox} further applied \algname{ProxSkip} to complex splitting schemes involving the sum of three operators in a forward-backward setting. Additionally, \citet{sadiev2022communication} and \citet{maranjyan2022gradskip} improved the computational complexity of \algname{ProxSkip} while preserving its communication efficiency. \citet{condat2023tamuna} introduced accelerated Local Training methods allowing client sampling based on \algname{ProxSkip}, while \citet{grudzien2023can, grudzien2023improving} proposed an accelerated method using the \algname{RandProx} approach with primal and dual updates.

In practice, Federated Learning faces a fundamental challenge: it is often infeasible for all clients to communicate and aggregate updates with the central server simultaneously due to limitations such as network bandwidth, client availability, or resource constraints. Therefore, rather than requiring all clients to participate in every round of communication, we adopt a strategy in which only a randomly selected subset of clients is involved in each aggregation step. This approach relies on uniform sampling of the clients, ensuring that the selection process is unbiased over time.

The method operates as follows: in each communication round, the central server sends the current global model, denoted by \(W^t\), along with the sampled matrix, to the clients chosen to participate in the current cohort. Each client in this cohort trains a local learnable matrix using an optimization algorithm (e.g., stochastic gradient descent) based on their local data. After completing the local updates, the clients send their computed updates (i.e., changes in model parameters) back to the central server.

Once the server receives these updates, it aggregates them (e.g., by averaging the updates) to produce an updated global model. In addition to the aggregation, the server may perform an additional server-side update step to further refine the model before broadcasting it in the next round. This iterative process of local training, communication, and aggregation continues until convergence is achieved or a predefined stopping criterion is met.

The proof is provided for Left Sketch (Definition \ref{def:left}). The result for Right Sketch (Definition \ref{def:right}) can be derived by following the same steps.

For local optimzier we use Random Reshuffling, where the effective step has a form:
\begin{align}
\label{eq:fed-rr}
    W^t_{m,i+1} = W^t_{m,i} - \gamma H^t_B \nabla f_{m,i} (W^t_{m,i})
\end{align}
The server-side step looks like 
$W^{t+1} = W^t - \tilde{\eta} H^t_B \frac{1}{C} \sum_{m\in S^t} \hat{A}^t_m.  $

Let us formulate nesesary assumptions
\begin{assumption}[Functional dissimilarity]
     The variance at the optimum in the non-convex regime is defined as
$$
\Delta^\star \stackrel{\text { def }}{=} f^\star-\frac{1}{M} \sum_{m=1}^M f_{m}^\star
$$
where $f_{m}^\star=\inf _W f_m(W)$ and $f^\star=\inf _W f(W)$. For each device $m$, the variance at the optimum is defined as
$$
\Delta_{m}^\star \stackrel{\text { def }}{=} f^\star-\frac{1}{n} \sum_{i=1}^n f_{m,i}^\star
$$
where $f_{m,i}^\star=\inf _W f_{m,i}(W)$
\end{assumption}

\subsection{Analysis of general non-convex setting}
\label{sec:Fed-gen}

\begin{theorem}
\label{thm:Fed-gen}
Suppose that Assumption \ref{asm:L-smooth} and Assumption \ref{asm:lambda} hold. Suppose that stepsizes $\gamma, \tilde{\eta} > 0$ is chosen such that $\gamma n \leq \tilde{\eta} \leq \frac{1-\lambda_{\min}^{H_B}}{4L} $. Then, the iterate ${W}^T$ of \algname{Fed-RAC-LoRA} method (Algorithm~\ref{alg:Fed-RAC-LoRA}) with \algname{RR} updates (Equation \ref{eq:fed-rr}) satisfy

\begin{align*}
 \min _{t=0, \ldots, T-1} \mathbb{E}\left[\left\|\nabla f\left(W^t\right)\right\|^2\right] \leq &\frac{4\left(1+4 \tilde{\eta}  L^3  \gamma^2 N^2  + 2L^2 \tilde{\eta} ^2 \frac{M-C}{C \max \{M-1,1\}} \right)^T}{\lambda_{\max}\left[ \mathbb{E}\left[ I-H^t_B \right] \right]\tilde{\eta}  T} \left(f(W^0) - f^\star\right)\\
&+\frac{8 \gamma^2 N L^3}{\lambda_{\max}\left[ \mathbb{E}\left[ I-H^t_B \right] \right]}\left(\frac{1}{M} \sum_{m=1}^M \Delta^{*}_m+N \Delta^*\right)\\
&+\frac{8 L^2 \tilde{\eta}}{\lambda_{\max}\left[ \mathbb{E}\left[ I-H^t_B \right] \right]}  \frac{M-C}{C \max \{M-1,1\}} \Delta^* .
\end{align*}

\end{theorem}

\begin{proof}

We start from $L$-smoothness:
    \begin{align*}
     &   f(W^{t+1})  \leq  f(W^t) + \left\langle \nabla f(W^t), W^{t+1} - W^t \right\rangle + \frac{L}{2}\left\| W^{t+1} - W^t \right\|^2\\
         \leq & f(W^t) - \left\langle \nabla f(W^t), \tilde{\eta} \frac{1}{CN} \sum_{m\in S^t} \sum^{N-1}_{i=0} H^t \nabla f^{\pi^t_{m,i}}_m\left( W^t_{m,i} \right) \right\rangle\\
        &+ \frac{L}{2}\left\| \tilde{\eta}  \frac{1}{CN}\sum_{m\in S^t} \sum^{N-1}_{i=0} H^t \nabla f_m^{\pi^t_{m,i}} (W^t_{m,i}) \right\|^2.
    \end{align*}
Now we take expectation with respect to sampling:
    \begin{align*}
    & \mathbb{E}_{S^t}  \left[ f(W^{t+1}) \right]
         \leq  f(W^t) - \tilde{\eta}  \mathbb{E}_{S^t} \left[ \left\langle \nabla f(W^t),  \frac{1}{CN} \sum_{m\in S^t} \sum^{N-1}_{i=0} H^t_B \nabla f^{\pi^t_{m,i}}_m\left( W^t_{m,i} \right) \right\rangle\right]\\
        &+ \frac{L}{2}\tilde{\eta} ^2 \mathbb{E}_{S^t} \left[ \left\|  \frac{1}{CN}\sum_{m\in S^t} \sum^{N-1}_{i=0} H^t_B \nabla f_m^{\pi^t_{m,i}} (W^t_{m,i}) \right\|^2\right]\\
        \leq & f(W^t) - \tilde{\eta}   \left\langle \nabla f(W^t),  \frac{1}{MN} \sum^M_{m=1} \sum^{N-1}_{i=0} H^t_B \nabla f^{\pi^t_{m,i}}_m\left( W^t_{m,i} \right) \right\rangle\\
        &+ \frac{L}{2}\tilde{\eta} ^2 \mathbb{E}_{S^t} \left[ \left\|  \frac{1}{CN}\sum_{m\in S^t} \sum^{N-1}_{i=0} H^t_B \nabla f_m^{\pi^t_{m,i}} (W^t_{m,i}) \right\|^2\right].
    \end{align*}
    Using $2 \left\langle a, b \right\rangle = \left\|a + b\right\|^2 - \left\|a\right\|^2 - \left\|b\right\|^2$, we have
        \begin{align*}
   &  \mathbb{E}_{S^t} \left[  f(W^{t+1}) \right]
        \leq  f(W^t) - \frac{\tilde{\eta} }{2} \| \nabla f(W^t) \|^2 - \frac{\tilde{\eta} }{2} \left\| \frac{1}{MN} \sum_{m=1}^M \sum^{N-1}_{i=0} H^t_B \nabla f_m^{\pi^t_{m,i}} (W^t_{m,i}) \right\|^2 \\
        &+\frac{\tilde{\eta} }{2}\left\| \nabla f(W^t) - \frac{1}{MN} \sum^M_{m=1}\sum^{N-1}_{i=0} H^t_B \nabla f_m^{\pi^t_{m,i}} (W^t_{m,i}) \right\|^2\\
        &+ \frac{L}{2}\tilde{\eta} ^2 \mathbb{E}_{S^t} \left[ \left\|  \frac{1}{CN}\sum_{m\in S^t} \sum^{N-1}_{i=0} H^t_B \nabla f_m^{\pi^t_{m,i}} (W^t_{m,i}) \right\|^2\right]\\
         \leq & f(W^t) - \frac{\tilde{\eta} }{2} \| \nabla f(W^t) \|^2 +\frac{\tilde{\eta} }{2}\left\| \nabla f(W^t) - \frac{1}{MN} \sum^M_{m=1}\sum^{N-1}_{i=0} H^t_B \nabla f_m^{\pi^t_{m,i}} (W^t_{m,i}) \right\|^2\\
        &+ \frac{L}{2}\tilde{\eta} ^2 \mathbb{E}_{S^t} \left[ \left\|  \frac{1}{CN}\sum_{m\in S^t} \sum^{N-1}_{i=0} H^t_B \nabla f_m^{\pi^t_{m,i}} (W^t_{m,i}) \right\|^2\right].
    \end{align*}
    Now we need to add and subtract $H^t_B \nabla f(W^t)$:
            \begin{align*}
     &\mathbb{E}_{S^t}  \left[  f(W^{t+1}) \right]
         \leq  f(W^t) - \frac{\tilde{\eta} }{2} \| \nabla f(W^t) \|^2\\
         &+\frac{\tilde{\eta} }{2}\left\| \nabla f(W^t) - H^t_B\nabla f(W^t) + H^t_B \nabla f(W^t) - \frac{1}{MN} \sum^M_{m=1}\sum^{N-1}_{i=0} H^t_B \nabla f_m^{\pi^t_{m,i}} (W^t_{m,i}) \right\|^2\\
        &+ \frac{L}{2}\tilde{\eta} ^2 \mathbb{E}_{S^t} \left[ \left\|  \frac{1}{CN}\sum_{m\in S^t} \sum^{N-1}_{i=0} H^t_B \nabla f_m^{\pi^t_{m,i}} (W^t_{m,i}) \right\|^2\right]\\
        \leq & f(W^t) - \frac{\tilde{\eta} }{2} \| \nabla f(W^t) \|^2\\
         &+\frac{\tilde{\eta} }{2}\left\| \nabla f(W^t)\left( I - H^t_B\right) + \frac{1}{MN} \sum^M_{m=1}\sum^{N-1}_{i=0} H^t_B \nabla f_m^{\pi^t_{m,i}} (W^t)  - \frac{1}{MN} \sum^M_{m=1}\sum^{N-1}_{i=0} H^t_B \nabla f_m^{\pi^t_{m,i}} (W^t_{m,i}) \right\|^2\\
        &+ \frac{L}{2}\tilde{\eta} ^2 \mathbb{E}_{S^t} \left[ \left\|  \frac{1}{CN}\sum_{m\in S^t} \sum^{N-1}_{i=0} H^t_B \nabla f_m^{\pi^t_{m,i}} (W^t_{m,i}) \right\|^2\right]\\
        \leq & f(W^t) - \frac{\tilde{\eta} }{2} \| \nabla f(W^t) \|^2\\
         &+\frac{\tilde{\eta} }{2}\left\| \nabla f(W^t)\left( I - H^t_B\right) + \frac{1}{MN} \sum^M_{m=1}\sum^{N-1}_{i=0} H^t_B \left( \nabla f_m^{\pi^t_{m,i}} (W^t) - \nabla f_m^{\pi^t_{m,i}} (W^t_{m,i}) \right)  \right\|^2\\
        &+ \frac{L}{2}\tilde{\eta} ^2 \mathbb{E}_{S^t} \left[ \left\|  \frac{1}{CN}\sum_{m\in S^t} \sum^{N-1}_{i=0} H^t_B \nabla f_m^{\pi^t_{m,i}} (W^t_{m,i}) \right\|^2\right].
    \end{align*}
Since $H^t_B(I-H^t_B) = 0$ we obtain
            \begin{align*}
   &  \mathbb{E}_{S^t}  \left[  f(W^{t+1}) \right] \leq  f(W^t) - \frac{\tilde{\eta} }{2} \| \nabla f(W^t) \|^2\\
         &+\frac{\tilde{\eta} }{2}\left\| \nabla f(W^t)\left( I - H^t_B\right) \right\|^2 + \frac{\tilde{\eta} }{2} \left\| \frac{1}{MN} \sum^M_{m=1}\sum^{N-1}_{i=0} H^t_B \left( \nabla f_m^{\pi^t_{m,i}} (W^t) - \nabla f_m^{\pi^t_{m,i}} (W^t_{m,i}) \right)  \right\|^2\\
        &+ \frac{L}{2}\tilde{\eta} ^2 \mathbb{E}_{S^t} \left[ \left\|  \frac{1}{CN}\sum_{m\in S^t} \sum^{N-1}_{i=0} H^t_B \nabla f_m^{\pi^t_{m,i}} (W^t_{m,i}) \right\|^2\right].
     \end{align*}
     Now we take conditional expectation and use tower property:
     \begin{align*}
              \mathbb{E}\left[f(W^{t+1})\mid W^t\right] \leq & f(W^t) - \frac{\tilde{\eta} }{2}\left\| \nabla f(W^t) \right\|^2 + \frac{\tilde{\eta} }{2}\mathbb{E}\left[ \| \nabla f(W^t)\left( I - H^t_B\right) \|^2\mid W^t\right]\\
              &+ \frac{\tilde{\eta} }{2} \mathbb{E}\left[\left\| \frac{1}{MN} \sum^M_{m=1}\sum^{N-1}_{i=0} H^t_B \left( \nabla f_m^{\pi^t_{m,i}} (W^t) - \nabla f_m^{\pi^t_{m,i}} (W^t_{m,i}) \right)  \right\|^2\mid W^t\right]\\
              &+ \frac{L}{2}\tilde{\eta} ^2 \mathbb{E}\left[ \mathbb{E}_{S^t} \left[ \left\|  \frac{1}{CN}\sum_{m\in S^t} \sum^{N-1}_{i=0} H^t_B \nabla f_m^{\pi^t_{m,i}} (W^t_{m,i}) \right\|^2\right]\mid W^t\right].
     \end{align*}
Next, we use eigenvalues to obtain bounds:
         \begin{align*}
              \mathbb{E} & \left[f(W^{t+1})\mid W^t\right] \leq  f(W^t) - \frac{\tilde{\eta} }{2}\left\| \nabla f(W^t) \right\|^2 + \frac{\tilde{\eta} }{2}\lambda_{\max}\left[ \mathbb{E}\left[I-H^t_B\right] \right] \|\nabla f(W^t) \|^2 \\
              &+ \!\frac{\tilde{\eta} }{2}  \mathbb{E}\!\left[\!\lambda_{\max}\!\left[H^t_B\right]\!\left\| \frac{1}{MN}\! \sum^M_{m=1}\sum^{N-1}_{i=0} \! \left( \!\nabla f_m^{\pi^t_{m,i}} (W^t) \!-\! \nabla f_m^{\pi^t_{m,i}} (W^t_{m,i}) \!\right)  \right\|^2\!\!\mid \!W^t\!\right]\\
              &+ \frac{L}{2}\tilde{\eta} ^2 \mathbb{E}\left[ \mathbb{E}_{S^t} \left[ \lambda_{\max}\left[H^t_B\right]\left\|  \frac{1}{Cn}\sum_{m\in S^t} \sum^{N-1}_{i=0}  \nabla f_m^{\pi^t_{m,i}} (W^t_{m,i}) \right\|^2\right]\mid W^t\right].
     \end{align*}
Since $\lambda_{\max}\left[ H^t_B \right]=1$ we have 
         \begin{align*}
              \mathbb{E}\left[f(W^{t+1})\mid W^t\right] \leq & f(W^t) - \frac{\tilde{\eta} }{2}\left\| \nabla f(W^t) \right\|^2 + \frac{\tilde{\eta} }{2}\lambda_{\max}\left[ \mathbb{E}\left[I-H^t_B\right] \right] \|\nabla f(W^t) \|^2 \\
              &+ \frac{\tilde{\eta} }{2}  \mathbb{E}\left[\left\| \frac{1}{MN} \sum^M_{m=1}\sum^{N-1}_{i=0}  \left( \nabla f_m^{\pi^t_{m,i}} (W^t) - \nabla f_m^{\pi^t_{m,i}} (W^t_{m,i}) \right)  \right\|^2\mid W^t\right]\\
              &+ \frac{L}{2}\tilde{\eta} ^2 \mathbb{E}\left[ \mathbb{E}_{S^t} \left[ \left\|  \frac{1}{CN}\sum_{m\in S^t} \sum^{N-1}_{i=0}  \nabla f_m^{\pi^t_{m,i}} (W^t_{m,i}) \right\|^2\right]\mid W^t\right].
     \end{align*}
     Using Lemma 5 from \citet{malinovsky2023server} we have 

\begin{align*}
\frac{L}{2}&\tilde{\eta} ^2 \mathbb{E}\left[ \mathbb{E}_{S^t} \left[ \left\|  \frac{1}{CN}\sum_{m\in S^t} \sum^{N-1}_{i=0}  \nabla f_m^{\pi^t_{m,i}} (W^t_{m,i}) \right\|^2\right]\mid W^t\right]\\
& \quad\quad\leq L^3 \tilde{\eta} ^2 \mathbb{E}\left[\frac{1}{M n} \sum_{m=1}^M \sum_{i=0}^{N-1}\left\|W_{m,i}^t-W^t\right\|^2\mid W^t\right]\\
&\qquad\qquad+L \tilde{\eta} ^2\left\|\nabla f\left(W^t\right)\right\|^2 \\ & \qquad\qquad+L \tilde{\eta} ^2 \frac{M-C}{C \max \{M-1,1\}}\left(2 L\left(f\left(W^t\right)-f^\star\right)+2 L \Delta^*\right)
\end{align*}
Using this bound and $L$-smoothness for the term in second line we obtain:
         \begin{align*}
          &    \mathbb{E}\left[f(W^{t+1})\mid W^t\right] \leq  f(W^t) - \frac{\tilde{\eta} }{2}\left\| \nabla f(W^t) \right\|^2 + \frac{\tilde{\eta} }{2}\lambda_{\max}\left[ \mathbb{E}\left[I-H^t_B\right] \right] \|\nabla f(W^t) \|^2 \\
              &+ \frac{\tilde{\eta} }{2} L^2 \mathbb{E}\left[ \frac{1}{Mn} \sum^M_{m=1}\sum^{n-1}_{i=0}  \left\|  W^t - W^t_{m,i} \right\|^2\mid W^t\right]\\
              &+L^3 \tilde{\eta} ^2 \mathbb{E}\left[ \frac{1}{Mn} \sum^M_{m=1}\sum^{n-1}_{i=0}  \left\|  W^t - W^t_{m,i} \right\|^2\mid W^t\right]+L \tilde{\eta} ^2\left\|\nabla f\left(W^t\right)\right\|^2 \\ & +L \tilde{\eta} ^2 \frac{M-C}{C \max \{M-1,1\}}\left(2 L\left(f\left(W^t\right)-f^\star\right)+2 L \Delta^*\right)
     \end{align*}
Since $\tilde{\eta} \leq \frac{1}{2L}$ we get 
         \begin{align*}
          &    \mathbb{E}\left[f(W^{t+1})\mid W^t\right] \leq  f(W^t) - \frac{\tilde{\eta} }{2}\left\| \nabla f(W^t) \right\|^2 + \frac{\tilde{\eta} }{2}\lambda_{\max}\left[ \mathbb{E}\left[I-H^t_B\right] \right] \|\nabla f(W^t) \|^2 \\
              &+ \tilde{\eta}  L^2 \mathbb{E}\left[ \frac{1}{MN} \sum^M_{m=1}\sum^{N-1}_{i=0}  \left\|  W^t - W^t_{m,i} \right\|^2\mid W^t\right]+L \tilde{\eta} ^2\left\|\nabla f\left(W^t\right)\right\|^2 \\ & +L \tilde{\eta} ^2 \frac{M-C}{C \max \{M-1,1\}}\left(2 L\left(f\left(W^t\right)-f^\star\right)+2 L \Delta^*\right).
     \end{align*}
Using lemma 6 from (cite) we obtain 
\begin{align*}
    \frac{1}{M N} \sum_{m=1}^M \sum_{i=0}^{N-1} \mathbb{E}\left[\left\|W^t-W_{m,i}^t\right\|^2 \mid W^t\right] \leq 4 \gamma^2 N^2 L\left(f\left(W^t\right)-f^*\right)\\+2 \gamma^2 N^2 L \Delta^*+2 \gamma^2 N L \frac{1}{M} \sum_{m=1}^M \Delta^{*}_m.
\end{align*}
Plugging this bound we obtain 
         \begin{align*}
        &      \mathbb{E}\left[f(W^{t+1})\mid W^t\right] \leq  f(W^t) - \frac{\tilde{\eta} }{2}\left\| \nabla f(W^t) \right\|^2 + \frac{\tilde{\eta} }{2}\lambda_{\max}\left[ \mathbb{E}\left[I-H^t_B\right] \right] \|\nabla f(W^t) \|^2 \\
              &+ \tilde{\eta}  L^2 \left(4 \gamma^2 N^2 L\left(f\left(W^t\right)-f^*\right)+2 \gamma^2 N^2 L \Delta^*+2 \gamma^2 N L \frac{1}{M} \sum_{m=1}^M \Delta^{*}_m\right)\\
&+L \tilde{\eta} ^2\left\|\nabla f\left(W^t\right)\right\|^2 \\ & +L \tilde{\eta} ^2 \frac{M-C}{C \max \{M-1,1\}}\left(2 L\left(f\left(W^t\right)-f^\star\right)+2 L \Delta^*\right).
     \end{align*}
Next, we have 
\begin{align*}
      &            \mathbb{E}\left[f(W^{t+1})\mid W^t\right] \leq  f(W^t) - \frac{\tilde{\eta} }{2}\left\| \nabla f(W^t) \right\|^2 \left( 1 - \lambda_{\max}\left[ \mathbb{E}\left[ I-H^t_B \right] \right] - 2L\tilde{\eta} \right)  \\
              &+ \tilde{\eta}  L^2 \left(4 \gamma^2 N^2 L\left(f\left(W^t\right)-f^*\right)+2 \gamma^2 N^2 L \Delta^*+2 \gamma^2 N L \frac{1}{M} \sum_{m=1}^M \Delta^{*}_m\right)\\
              & +L \tilde{\eta} ^2 \frac{M-C}{C \max \{M-1,1\}}\left(2 L\left(f\left(W^t\right)-f^\star\right)+2 L \Delta^*\right).
\end{align*}
Using $\tilde{\eta} \leq\frac{1 - \lambda_{\max}\left[ \mathbb{E}\left[ H^t_B \right] \right]}{4L}$ we get
\begin{align*}
          &        \mathbb{E}\left[f(W^{t+1})\mid W^t\right] \leq  f(W^t) - \frac{\tilde{\eta} }{4}\left\| \nabla f(W^t) \right\|^2 \left( 1 - \lambda_{\max}\left[ \mathbb{E}\left[ I-H^t_B \right] \right]\right)  \\
              &+ \tilde{\eta}  L^2 \left(4 \gamma^2 N^2 L\left(f\left(W^t\right)-f^*\right)+2 \gamma^2 N^2 L \Delta^*+2 \gamma^2 N L \frac{1}{M} \sum_{m=1}^M \Delta^{*}_m\right)\\
              & +L \tilde{\eta} ^2 \frac{M-C}{C \max \{M-1,1\}}\left(2 L\left(f\left(W^t\right)-f^\star\right)+2 L \Delta^*\right).
\end{align*}
Next, we subtract $f^\star$ from both sides:
\begin{align*}
                 &     \mathbb{E}\left[f(W^{t+1})\mid W^t\right] - f^\star \leq f(W^t) - f^\star - \frac{\tilde{\eta} }{4}\left\| \nabla f(W^t) \right\|^2 \left( 1 - \lambda_{\max}\left[ \mathbb{E}\left[ I-H^t_B \right] \right]\right)  \\
              &+ \tilde{\eta}  L^2 \left(4 \gamma^2 N^2 L\left(f\left(W^t\right)-f^*\right)+2 \gamma^2 N^2 L \Delta^*+2 \gamma^2 N L \frac{1}{M} \sum_{m=1}^M \Delta^{*}_m\right)\\
              & +L \tilde{\eta} ^2 \frac{M-C}{C \max \{M-1,1\}}\left(2 L\left(f\left(W^t\right)-f^\star\right)+2 L \Delta^*\right).
\end{align*}
Taking full expectation we obtain
\begin{align*}
              &        \mathbb{E}\left[f(W^{t+1}) - f^\star \right] \leq  \mathbb{E}\left[f(W^t) - f^\star\right] \left(1+4 \tilde{\eta}  L^3  \gamma^2 N^2  + 2L^2 \tilde{\eta} ^2 \frac{M-C}{C \max \{M-1,1\}} \right)\\
                      &- \frac{\tilde{\eta} }{4}\left\| \nabla f(W^t) \right\|^2 \left( 1 - \lambda_{\max}\left[ \mathbb{E}\left[ I-H^t_B \right] \right]\right)  \\
              &+ \tilde{\eta}  L^2 \left(2 \gamma^2 N^2 L \Delta^*+2 \gamma^2 N L \frac{1}{M} \sum_{m=1}^M \Delta^{*}_m\right) +2L^2 \tilde{\eta} ^2 \frac{M-C}{C \max \{M-1,1\}} \Delta^*.
\end{align*}
Next, we apply lemma from \citet{khaled2020better} and obtain
\begin{align*}
 \min _{t=0, \ldots, T-1} \mathbb{E}\left[\left\|\nabla f\left(W^t\right)\right\|^2\right] \leq &\frac{4\left(1+4 \tilde{\eta}  L^3  \gamma^2 N^2  + 2L^2 \tilde{\eta} ^2 \frac{M-C}{C \max \{M-1,1\}} \right)^T}{\lambda_{\max}\left[ \mathbb{E}\left[ I-H^t_B \right] \right]\tilde{\eta}  T} \left(f(W^0) - f^\star\right)\\
&+\frac{8 \gamma^2 N L^3}{\lambda_{\max}\left[ \mathbb{E}\left[ I-H^t_B \right] \right]}\left(\frac{1}{M} \sum_{m=1}^M \Delta^{*}_m+N \Delta^*\right)\\
&+\frac{8 L^2 \tilde{\eta}}{\lambda_{\max}\left[ \mathbb{E}\left[ I-H^t_B \right] \right]}  \frac{M-C}{C \max \{M-1,1\}} \Delta^* .
\end{align*}
\end{proof}

\subsection{Analysis of Polyak-Łojasiewicz setting }
\label{sec:Fed-PL}
\begin{theorem}

Suppose that Assumption \ref{asm:L-smooth}, Assumption \ref{asm:PL} and Assumption \ref{asm:lambda} hold. Suppose that stepsizes $\gamma, \tilde{\eta} > 0$ is chosen such that $\gamma n \leq \tilde{\eta} \leq \frac{1-\lambda_{\min}^{H_B}}{4L} $. Then, the iterate ${W}^T$ of \algname{Fed-RAC-LoRA} method (Algorithm \ref{alg:Fed-RAC-LoRA}) with \algname{RR} updates (Equation \ref{eq:fed-rr}) satisfy
        \begin{align*}
                &  \mathbb{E}\left[f(W^{t+1}) - f^\star\right] \leq  (f(W^0) - f^\star)\left(1 - \tilde{\eta} \mu \left( 1 - \lambda_{\max}\left[ \mathbb{E}\left[ I-H^t \right] \right] - 3L\tilde{\eta}\right) \right)^T   \\
              &+ \frac{\tilde{\eta} L^2}{\tilde{\eta} \mu \left( 1 - \lambda_{\max}\left[ \mathbb{E}\left[ I-H^t \right] \right] - 3L\tilde{\eta}\right)} \left(2 \gamma^2 N^2 L \Delta^*+2 \gamma^2 N L \frac{1}{M} \sum_{m=1}^M \Delta^{*}_m\right)\\
              & + \frac{L \tilde{\eta}^2}{\tilde{\eta} \mu \left( 1 - \lambda_{\max}\left[ \mathbb{E}\left[ I-H^t \right] \right] - 3L\tilde{\eta}\right)} \frac{M-C}{C \max \{M-1,1\}}\left(2 L\left(f\left(W^t\right)-f^\star\right)+2 L \Delta^*\right).
\end{align*}
\end{theorem}

\begin{proof}
    We start from 
    \begin{align*}
                  \mathbb{E}\left[f(W^{t+1})\mid W^t\right] \leq & f(W^t) - \frac{\tilde{\eta}}{2}\left\| \nabla f(W^t) \right\|^2 \left( 1 - \lambda_{\max}\left[ \mathbb{E}\left[ I-H^t \right] \right] - 2L\tilde{\eta}\right)  \\
              &+ \tilde{\eta} L^2 \left(4 \gamma^2 N^2 L\left(f\left(W^t\right)-f^*\right)+2 \gamma^2 N^2 L \Delta^*+2 \gamma^2 N L \frac{1}{M} \sum_{m=1}^M \Delta^{*}_m\right)\\
              & +L \tilde{\eta}^2 \frac{M-C}{C \max \{M-1,1\}}\left(2 L\left(f\left(W^t\right)-f^\star\right)+2 L \Delta^*\right).
\end{align*}
Using Assumption\ref{asm:PL} we have 
    \begin{align*}
                  \mathbb{E}\left[f(W^{t+1})\mid W^t\right] \leq & f(W^t) - \tilde{\eta} \mu\left\| \nabla f(W^t) \right\|^2 \left( 1 - \lambda_{\max}\left[ \mathbb{E}\left[ I-H^t \right] \right] - 2L\tilde{\eta}\right) \left( f(W^t) - f^\star \right)  \\
              &+ \tilde{\eta} L^2 \left(4 \gamma^2 n^2 L\left(f\left(W^t\right)-f^*\right)+2 \gamma^2 N^2 L \Delta^*+2 \gamma^2 N L \frac{1}{M} \sum_{m=1}^M \Delta^{*}_m\right)\\
              & +L \tilde{\eta}^2 \frac{M-C}{C \max \{M-1,1\}}\left(2 L\left(f\left(W^t\right)-f^\star\right)+2 L \Delta^*\right).
\end{align*}
Using the stepsize $\gamma \leq \frac{1}{4nL}$ we have 

    \begin{align*}
                  \mathbb{E}\left[f(W^{t+1})\mid W^t\right] \leq & f(W^t) - \tilde{\eta} \mu \left( 1 - \lambda_{\max}\left[ \mathbb{E}\left[ I-H^t \right] \right] - 3L\tilde{\eta}\right) \left( f(W^t) - f^\star \right)  \\
              &+ \tilde{\eta} L^2 \left(2 \gamma^2 N^2 L \Delta^*+2 \gamma^2 N L \frac{1}{M} \sum_{m=1}^M \Delta^{*}_m\right)\\
              & +L \tilde{\eta}^2 \frac{M-C}{C \max \{M-1,1\}}\left(2 L\left(f\left(W^t\right)-f^\star\right)+2 L \Delta^*\right).
\end{align*}
After unrolling the recursion we obtain
    \begin{align*}
                &  \mathbb{E}\left[f(W^{t+1}) - f^\star\right] \leq  (f(W^0) - f^\star)\left(1 - \tilde{\eta} \mu \left( 1 - \lambda_{\max}\left[ \mathbb{E}\left[ I-H^t \right] \right] - 3L\tilde{\eta}\right) \right)^T   \\
              &+ \frac{\tilde{\eta} L^2}{\tilde{\eta} \mu \left( 1 - \lambda_{\max}\left[ \mathbb{E}\left[ I-H^t \right] \right] - 3L\tilde{\eta}\right)} \left(2 \gamma^2 N^2 L \Delta^*+2 \gamma^2 N L \frac{1}{M} \sum_{m=1}^M \Delta^{*}_m\right)\\
              & + \frac{L \tilde{\eta}^2}{\tilde{\eta} \mu \left( 1 - \lambda_{\max}\left[ \mathbb{E}\left[ I-H^t \right] \right] - 3L\tilde{\eta}\right)} \frac{M-C}{C \max \{M-1,1\}}\left(2 L\left(f\left(W^t\right)-f^\star\right)+2 L \Delta^*\right).
\end{align*}

\end{proof}

\end{document}